\documentclass[a4paper,12pt]{article}

\usepackage[utf8]{inputenc}
\usepackage[unicode=true,hypertexnames=false]{hyperref}
\usepackage{amsmath,amssymb,amsthm,amsfonts}
\usepackage{dsfont}
\usepackage{mathtools} 
\usepackage[linesnumbered,ruled]{algorithm2e}
\usepackage{xspace}
\usepackage[dvipsnames,table]{xcolor}
\usepackage{tikz}
\usepackage{multirow}

\usepackage{pdflscape}
\usepackage{afterpage}

\usepackage{pgfplots}
\pgfplotsset{compat=newest}
\pgfplotscreateplotcyclelist{myplotcycle}{%
blue,dashed\\     
blue\\            
red!80!black\\    
violet!80!black\\ 
orange\\          
green!70!black\\  
cyan!80!black\\   
black,dashed\\    
black\\           
red!80!black,dashdotted\\    
violet!80!black,dashdotted\\ 
orange,dashdotted\\          
}
\pgfplotscreateplotcyclelist{threeheavycycle}{%
blue\\            
red!80!black\\     
violet!80!black\\ 
orange\\          
green!70!black\\  
cyan!80!black\\   
brown\\
black,dashed\\    
}
\pgfmathdeclarefunction{f1}{1}{%
  \pgfmathparse{min(#1^(1.5), 8.05)}%
}
\pgfmathdeclarefunction{f2}{1}{%
  \pgfmathparse{max((#1 - 1)^(1.5), 0.95)}%
}
\pgfmathdeclarefunction{f3}{1}{%
  \pgfmathparse{min((1.01, (#1 - 1)^(-1))}%
}
\pgfmathdeclarefunction{f4}{1}{%
  \pgfmathparse{max(#1^(-1), 0.24)}%
}
\pgfmathdeclarefunction{f5}{1}{%
  \pgfmathparse{min(#1^(3/2), 12*(#1 - 1)^(-1/2), 6.45)}%
}
\pgfmathdeclarefunction{f6}{1}{%
  \pgfmathparse{min((#1 - 1)^(3/2), 12*(#1)^(-1/2)))}%
}
\pgfmathdeclarefunction{f7}{1}{%
  \pgfmathparse{max((#1 - 1)^(3/2), 12*(#1)^(-1/2)))}%
}

\newcommand{\oea}{\mbox{$(1 + 1)$~EA}\xspace}

\newcommand{\mplea}{\mbox{$(\mu+\lambda)$~EA}\xspace}

\newcommand{\ollga}{${(1 + (\lambda , \lambda))}$~GA\xspace}

\newcommand{\onemax}{\textsc{OneMax}\xspace}
\newcommand{\leadingones}{\textsc{LeadingOnes}\xspace}
\newcommand{\om}{\textsc{OM}\xspace}
\newcommand{\jump}{\textsc{Jump}\xspace}
\newcommand{\maxsat}{\textsc{MAX-3SAT}\xspace}
\newcommand{\N}{{\mathbb N}}
\newcommand{\R}{{\mathbb R}}

\DeclareMathOperator{\Bin}{Bin}
\DeclareMathOperator{\pow}{pow}
\DeclareMathOperator{\Geom}{Geom}
\DeclareMathOperator*{\argmax}{arg\,max}

\newtheorem{theorem}{Theorem}
\newtheorem{lemma}[theorem]{Lemma}

\newtheorem{definition}[theorem]{Definition}

\begin{document}

\title{First Steps Towards a Runtime Analysis When Starting With a Good Solution}

\author{Denis Antipov \\
		The University of Adelaide \\
  		Adelaide, Australia \\
        \and
        Maxim Buzdalov\\
		Aberystwyth University \\
  		Aberystwyth, Wales, UK \\
		\and
		Benjamin Doerr \\
		Laboratoire d'Informatique (LIX), \\
		CNRS, \'Ecole Polytechnique, \\ 
		Institut Polytechnique de Paris \\
		Palaiseau, France \\
}

\maketitle

{\sloppy

\begin{abstract}
  The mathematical runtime analysis of evolutionary algorithms traditionally regards the time an algorithm needs to find a solution of a certain quality when initialized with a random population. In practical applications it may be possible to guess solutions that are better than random ones. We start a mathematical runtime analysis for such situations. We observe that different algorithms profit to a very different degree from a better initialization. We also show that the optimal parameterization of an algorithm can depend strongly on the quality of the initial solutions. To overcome this difficulty, self-adjusting and randomized heavy-tailed parameter choices can be profitable. Finally, we observe a larger gap between the performance of the best evolutionary algorithm we found and the corresponding black-box complexity. This could suggest that evolutionary algorithms better exploiting good initial solutions are still to be found. These first findings stem from analyzing the performance of the $(1+1)$ evolutionary algorithm and the static, self-adjusting, and heavy-tailed $(1 + (\lambda,\lambda))$ GA on the OneMax benchmark. We are optimistic that the question how to profit from good initial solutions is interesting beyond these first examples. 
%

\end{abstract}

\section{Introduction}

The mathematical runtime analysis  (see, e.g,.~\cite{AugerD11,DoerrN20,Jansen13,NeumannW10,ZhouYQ19}) has contributed to our understanding of evolutionary algorithms (EAs) via rigorous analyses how long an EA takes to optimize a particular problem. The overwhelming majority of these results considers a random or worst-case initialization of the algorithm. In this work, we argue that it also makes sense to analyze the runtime of algorithms starting already with good solutions. This is justified because such situations arise in practice and because, as we observe in this work, different algorithms show a different runtime behavior when started with such good solutions. In particular, we observe that the $(1 + (\lambda,\lambda))$ genetic algorithm (\ollga) profits from good initial solutions by much more than, e.g., the \oea. From a broader perspective, this work suggests that the recently proposed fine-grained runtime notions like fixed budget analysis~\cite{JansenZ14} and fixed target analysis~\cite{BuzdalovDDV22}, which consider optimization up to a certain solution quality, should be extended to also take into account different initial solution qualities. 

\subsection{Starting with Good Solutions}

As just said, the vast majority of the runtime analyses assume a random initialization of the algorithm or they prove performance guarantees that hold for all initializations (worst-case view). This is justified for two reasons. (i)~When optimizing a novel problem for which little problem-specific understanding is available, starting with random initial solutions is a recommended approach. This avoids that a wrong understanding of the problem leads to an unfavorable initialization. Also, with independent runs of the algorithm automatically reasonably diverse initializations are employed. (ii)~For many optimizations processes analyzed with mathematical means it turned out that there is not much advantage of starting with a good solution. For this reason, such results are not stated explicitly, but can often be derived from the proofs. For example, when optimizing the simple \onemax benchmark via the equally simple \oea, then results like~\cite{Muhlenbein92,DrosteJW02,DoerrFW11,DoerrDY20} show a very limited advantage from a good initialization. When starting with a solution having already 99\% of the maximal fitness,  the expected runtime has the same $en \ln(n) \pm O(n)$ order of magnitude as the one starting from a random solution. Hence the gain from starting with the good solution is bounded by an $O(n)$ lower order term. Even when starting with a solution of fitness $n - \sqrt n$, that is, with fitness distance $\sqrt n$ to the optimum of fitness $n$, then only a runtime reduction by asymptotically a factor of a half results. Clearly, a factor-two runtime improvement is interesting in practice, but the assumption that an initial solution can be found that differs from the optimum in only $\sqrt n$ of the $n$ bit positions, is very optimistic. Without going into details, we note that other problem which exhibit a multiplicative drift behavior~\cite{DoerrJW12algo}, such as the optimization of minimum spanning trees~\cite{NeumannW07}, Euler cycles~\cite{DoerrJ07gecco}, or shortest paths~\cite{BaswanaBDFKN09,DoerrJ10}, would similarly little profit from a good initialization. 

So there is some justification for random initializations, but we also see a number of situations in which better-than-random solutions are available (and this is the motivation of this work). The obvious one is that a problem is to be solved for which some, at least intuitive, understanding is available. This is a realistic assumption in scenarios where similar problems are to be solved over a longer time period or where problems are solved by combining a human understanding of the problem with randomized heuristics. A second situation in which we expect to start with a good solution is reoptimization. Reoptimization~\cite{SchieberSTT18,Zych18} means that we had already solved a problem, then a mild change of the problem data arises (due to a change in the environment, a customer being unhappy with a particular aspect of the solution, etc.), and we react to this change not by optimizing the new problem from scratch, but by initializing the EA with solutions that were good in the original problem. While there is a decent amount of runtime analysis literature on how EAs cope with dynamic optimization problems, see~\cite{NeumannPR20bookchapter}, almost all of them regard the situation that a dynamic change of the instance happens frequently and the question is how well the EA adjusts to these changes. The only mathematical runtime analysis of a true reoptimization problem we are aware of is~\cite{DoerrDN19}. The focus there, however, is to modify an existing algorithm so that it better copes with the situation that the algorithm is started with a solution that is structurally close to the optimum, but has a low fitness obscuring to the algorithm that the current solution is already structurally good. A third way optimization problems starting with a good solution can show up is when using a first heuristics for some time and then switching to a second heuristic. This is done frequently in practice, however, we are aware of only one rigorous result in this direction, namely that such a hybridization of the Metropolis algorithm and the \oea gives a better performance than the \oea alone~\cite{DoerrRW22arxiv}. 

We note that using a known good solution to initialize a randomized search heuristic is again a heuristic approach. It is intuitive that an iterative optimization heuristic can profit from such an initialization, but there is no guarantee and, clearly, there are also situations where using such initializations is detrimental. As one example, assume that we obtain good initial solutions from running a simple hill-climber. Then these initial solutions could be local optima which are very hard to leave. An evolutionary algorithm initialized with random solutions might find it easier to generate a sufficient diversity that allows to reach the basin of attraction of the optimum. So obviously some care is necessary when initializing a search heuristic with good solutions. Several practical applications of evolutionary algorithms have shown advantages of initializations with good solutions, e.g.,~\cite{Liaw00} on the open shop scheduling problem.

While there are no explicit mathematical runtime analyses for EAs starting with a good solution, it is clear that many of the classic results in their proofs reveal much information also on runtimes starting from a good solution. This is immediately clear for the fitness level method~\cite{Wegener01}, but also for drift arguments like~\cite{HeY01,DoerrG13algo,DoerrJW12algo,MitavskiyRC09,Johannsen10} when as potential function the fitness or a similar function is used, and for many other results. By not making these results explicit, however, it is hard to see the full picture and to draw the right conclusions. 

\subsection{The \ollga Starting with Good Solutions}

In this work, we make explicit how the \ollga optimizes \onemax when starting from a solution with fitness distance $D$ from the optimum. We observe that the \ollga profits in a much stronger way from such a good initialization than other known algorithms. For example, when starting in fitness distance $D = \sqrt n$, the expected time to find the optimum is only $\tilde O(n^{3/4})$ when using optimal parameters. We recall that this algorithm has a runtime of roughly $n \sqrt{\log n}$ when starting with a random solution~\cite{DoerrDE15,DoerrD18}. We recall further that the \oea has an expected runtime of $(1 \pm o(1)) \frac 12 e n \ln(n)$ when starting in fitness distance $\sqrt n$ and an expected runtime of $(1 \pm o(1)) e n \ln n$ when starting with a random solution. So clearly, the \ollga profits to a much higher degree from a good initialization than the $\oea$. We made this precise for the \oea, but it is clear from other works such as~\cite{JansenJW05,Witt06,DoerrK15,AntipovD21algo} that similar statements hold as well for many other \mplea{}s optimizing \onemax, at least for some ranges of the parameters.

The runtime stated above for the \ollga assumes that the algorithm is used with the optimal parameter setting, more precisely, with the optimal setting for starting with a solution of fitness-distance $D$. Besides that we usually do not expect the algorithm user to guess the optimal parameter values, it is also not very realistic to assume that the user has a clear picture on how far the initial solution is from the optimum. For that reason, we also regard two parameter-less variants of the \ollga (where parameterless means that parameters with a crucial influence on the performance are replaced by hyperparameters for which the influence is less critical or for which we can give reasonable general rules of thumb). 

Already in~\cite{DoerrDE15}, a self-adjusting choice based on the one-fifth success rule of the parameters of the \ollga was proposed. This was shown to give a linear runtime on \onemax in~\cite{DoerrD18}. We note that this is, essentially, a parameterless algorithm since the target success rate (the ``one-fifth'') and the update factor had only a small influence on the result provided that they were chosen not too large (where the algorithm badly fails). See~\cite[Section~6.4]{DoerrD18} for more details. For this algorithm, we show that it optimizes \onemax in time $O(\sqrt{nD})$ when starting in distance $D$. Again, this is a parameterless approach (when taking the previous recommendations on how to set the hyperparameters).

A second parameterless approach for the \ollga was recently analyzed in~\cite{AntipovBD22}, namely to choose the parameter $\lambda$ randomly from a power-law distribution. Such a heavy-tailed parameter choice was shown to give a performance only slightly below the one obtainable from the best instance-specific values for the \oea optimizing jump functions~\cite{DoerrLMN17}. Surprisingly, the \ollga with heavy-tailed parameter choice could not only overcome the need to specify parameter values, it even outperformed any static parameter choice and had the same $O(n)$ runtime that the self-adjusting \ollga had~\cite{AntipovBD22}. When starting with a solution in fitness distance $D$, this algorithm with any power-law exponent equal to or slightly above two gives a performance which is only by a small factor slower than $O(\sqrt{nD})$.

\subsection{Experimental Results}

We support our theoretical findings with an experimental validation, which shows that both the self-adjusting and the heavy-tailed version of the \ollga indeed show the desired asymptotic behavior and this with only moderate implicit constants. In particular, the one-fifth self-adjusting version can be seen as a very confident winner in all cases, and the heavy-tailed versions with different power-law exponents follow it with the accordingly distributed runtimes. Interestingly enough, the logarithmically-capped self-adjusting version, which has been shown to be beneficial for certain problems other than \onemax~\cite{BuzdalovD17} and just a tiny bit worse than the basic one-fifth version on \onemax, starts losing ground to the heavy-tailed versions at distances just slightly smaller than $\sqrt{n}$.

\subsection{Black-Box Complexity and Lower Bounds}

The results above show that some algorithms can profit considerably from good initial solutions (but many do not). This raises the question of how far we can go in this direction, or formulated inversely, what lower bounds on this runtime problem we can provide. Given an objective function from a given class of functions and a search point at a Hamming distance $D$ from the optimum, one can define the relevant flavor of the unrestricted black-box complexity as the smallest expected number of fitness evaluations that an otherwise unrestricted black-box algorithm performs to find the optimum (of a worst-case input). 

If the class of functions consists of all \onemax-type functions, that is, \onemax and all functions with an isomorphic fitness landscape, we show that the black-box complexity is $\Theta(\log_{1 + \min\{D, n-D\}} \binom{n}{D})$, which is $\Theta(\frac{D \log (n/D)}{\log D})$ assuming $D \le n/2$. The lower bound uses the classic argument via randomized search trees and Yao's minimax principle from~\cite{DrosteJW06}, with additional careful considerations for how many different answers one can get for each query. The upper bound uses the classic random guessing strategy of~\cite{erd63}, which we show to require at most twice as many evaluations as the lower bound.

For small $D$, this black-box complexity of order $\Theta(\frac{D \log (n/D)}{\log D})$ is considerably lower than our upper bounds. Also, this shows a much larger gap between black-box complexity and EA performance than in the case of random initialization, where the black-box complexity is $\Theta(\frac{n}{\log n})$ and simple EAs have an $O(n \log n)$ performance.

\subsection{Synopsis and Structure of the Paper}
Overall, our results show that the question of how EAs work when started with a good initial solution is far from trivial. Some algorithms profit more from this than others, the question of how to set the parameters might be influenced by the starting level $D$ and this may make parameterless approaches more important, and the larger gap to the black-box complexity could suggest that there is room for further improvements. 

The rest of the paper is organized as follows. 
In Section~\ref{sec:preliminaries} we formally define the considered algorithms and the problem and collect some useful analysis tools. In Section~\ref{sec:runtimes} we prove the upper bounds on the runtime of the algorithms and deliver general recommendations on how to use each algorithm. In Section~\ref{sec:bbc} we formally define the conditional unrestricted black-box complexity and prove first upper and lower bounds for that complexity. In Section~\ref{sec:experiments} we check how our recommendations work in experiments.

\section{Preliminaries}
\label{sec:preliminaries}



\subsection{The \ollga and Its Modifications}
\label{sec:algos}

We consider the \ollga, which is a genetic algorithm for the optimization of $n$-dimensional pseudo-Boolean functions, first proposed in~\cite{DoerrDE15}. This algorithm has three parameters, which are the mutation rate $p$, the crossover bias $c$, and the population size $\lambda$.

The \ollga stores the current individual $x$, which is initialized with a random bit string. Each iteration of the algorithm consists of a mutation phase and a crossover phase. In the mutation phase we first choose a number $\ell$ from the binomial distribution with parameters $n$ and $p$. Then we create $\lambda$ offsprings by flipping $\ell$ random bits in $x$, independently for each offspring. An offspring with the best fitness is chosen as the mutation winner $x'$ (all ties are broken uniformly at random). Note that $x'$ can and often will have a worse fitness than~$x$.

In the crossover phase we create $\lambda$ offspring by applying a biased crossover to $x$ and $x'$ (independently for each offspring). This biased crossover takes each bit from $x$ with probability $(1 - c)$ and from $x'$ with probability $c$. A crossover offspring with best fitness is selected as the crossover winner $y$ (all ties are broken uniformly at random). If $y$ is not worse than $x$, it replaces the current individual. The pseudocode of the \ollga is shown in Algorithm~\ref{alg:pseudo}. 

\begin{algorithm}[h]
    $x \gets $ random bit string of length $n$\;
    \While{not terminated}
        {
        \textbf{Mutation phase:}\\
        Choose $\ell \sim \Bin\left(n, p\right)$\;
        \For{$i \in [1..\lambda]$}
            {$x^{(i)} \gets$ a copy of $x$\;
            Flip $\ell$ bits in $x^{(i)}$ chosen uniformly at random\;
            }
        $x' \gets \argmax_{z \in \{x^{(1)}, \dots, x^{(\lambda)}\}}f(z)$\;
        \textbf{Crossover phase:}\\
        \For{$i \in [1..\lambda]$}
            {Create $y^{(i)}$ by taking each bit from $x'$ with probability $c$ and from $x$ with probability $(1 - c)$\;
            }
        $y \gets \argmax_{z \in \{y^{(1)}, \dots, y^{(\lambda)}\} }f(z)$\;
        \If{$f(y) \ge f(x)$}
            {
             $x \gets y$\;   
            }
        }
    \caption{The \ollga maximizing a pseudo-Boolean function~$f$.}
    \label{alg:pseudo}
\end{algorithm}

Based on intuitive considerations and rigorous runtime analyses, a standard parameter settings was proposed in which the mutation rate and crossover bias are defined via the population size, namely, $p = \frac{\lambda}{n}$ and $c = \frac{1}{\lambda}$.

It was shown in~\cite{DoerrDE15} that with a suitable \textbf{static parameter} value for $\lambda$, this algorithm can solve the \onemax function in $O(n\sqrt{\log(n)})$ fitness evaluations (this bound was minimally reduced and complemented with a matching lower bound in~\cite{DoerrD18}). The authors of~\cite{DoerrDE15} noticed that with the \textbf{fitness-dependent parameter $\lambda = \sqrt{\frac{n}{d}}$} the algorithm solves \onemax in only $\Theta(n)$ iterations. 

The fitness-depending parameter setting was not satisfying, since it is too problem-specific and most probably does not work on practical problems. For this reason, also a \textbf{self-adjusting parameter choice} for $\lambda$ was proposed in~\cite{DoerrDE15} and analyzed rigorously in~\cite{DoerrD18}. It uses a simple one-fifth rule, multiplying the parameter $\lambda$ by some constant $A > 1$ at the end of the iteration when $f(y) \le f(x)$, and dividing $\lambda$ by $A^4$ otherwise (the forth power ensures the desired property that the parameter does not change in the long run when in average one fifth of the iterations are successful). This simple rule was shown to keep the parameter $\lambda$ close to the optimal fitness-dependent value during the whole optimization process, leading to a $\Theta(n)$ runtime on \onemax. However, this method of parameter control was not efficient on the \maxsat problem, which has a lower fitness-distance correlation than \onemax~\cite{BuzdalovD17}. Therefore, capping the maximal value of $\lambda$ at $2\ln(n + 1)$ was needed to obtain a good performance on this problem.

Inspired by~\cite{DoerrLMN17}, the recent paper~\cite{AntipovBD22} proposed use a heavy-tailed random $\lambda$, which gave a birth to the \textbf{fast \ollga}. In this algorithm the parameter~$\lambda$ is chosen from the power-law distribution with exponent $\beta$ and with upper limit $u$, which we denote by $\pow(\beta, u)$. Here for all $i \in \N$ we have
\begin{align*}
    \Pr[\lambda = i] = \begin{cases}
        C_{\beta, u}i^{-\beta}, \text{ if } i \in [1..u], \\
        0, \text{ otherwise,}
    \end{cases}
\end{align*} 
where $C_{\beta, u} = (\sum_{j = 1}^u j^{-\beta})^{-1}$ is the normalization coefficient. It was proven that the fast \ollga finds the optimum of \onemax in $\Theta(n)$ fitness evaluations if $\beta \in (2, 3)$ and $u$ is large enough. Also it was empirically shown that this algorithm without further capping of $\lambda$ is quite efficient on \maxsat.

When talking about the runtime of the \ollga, we denote the number of iterations until the optimum is found by $T_I$ and the number of fitness evaluations until the optimum is found by $T_F$. We denote the distance of the current individual to the optimum by~$d$.

We note in passing the the \ollga was also analyzed on the \leadingones problem, where it was shown to have a performance comparable to the one of simpler evolutionary algorithms~\cite{AntipovDK19foga}. It was also analyzed on \jump functions, where it greatly outperformed simple algorithms, however only with a different parameter setting than the one that was successful~\cite{AntipovDK22,FajardoS22}.

\subsection{Problem Statement}
\label{sec:problem}

The main object of this paper is the runtime of the algorithms discussed in Section~\ref{sec:algos} when they start in distance $D$ from the optimum, where $D$ should be smaller than the distance of a random solution. For this purpose we consider the classic \onemax function, which is defined on the space of bit strings of length $n$ by
\begin{align*}
    \onemax(x) = \om(x) = \sum_{i = 1}^n x_i.
\end{align*}

In the context of black-box complexity, however, we are rather interested in a family of isomorphic \onemax-like problems defined as follows
\begin{align*}
    \om_z(x) = \sum_{i = 1}^n (1 - |x_i - z_i|),
\end{align*}
where $z \in \{0,1\}^n$ is the optimum. Clearly, all evolutionary algorithms considered in this paper are unbiased (in the sense of~\cite{LehreW12,RoweV11}), so they behave identically on all these problems regardless of the values of $z$.

%

\subsection{Probability for Progress}

To prove our upper bounds on the runtimes we use the following estimate for the probability that the \ollga finds a better solution in one iteration.

\begin{lemma}
    \label{lem:progress}
    The probability that $\om(y) > \om(x)$ is $\Omega(\min\{1, \frac{d\lambda^2}{n}\})$.
\end{lemma}

To prove this lemma we use the following auxiliary result from~\cite{AntipovBD22}, a slight adaptation of~\cite[Lemma~8]{RoweS14}.

\begin{lemma}[Lemma 2.2 in~\cite{AntipovBD22}]
	\label{lem:Bernoulli}
	For all $p \in [0, 1]$ and all $\lambda > 0$ we have
	\[
		1 - (1 - p)^\lambda \ge \frac{\lambda p}{1 + \lambda p}.
	\]
\end{lemma}

\begin{proof}[Proof of Lemma~\ref{lem:progress}]
    By Lemma~7 in~\cite{DoerrDE15} the probability to have a true progress in one iteration is $\Omega(1 - (\frac{n - d}{n})^\frac{\lambda^2}{2})$. By Lemma~\ref{lem:Bernoulli} this is at least $\Omega(\min\{1, \frac{d\lambda^2}{n}\})$.
\end{proof}

\subsection{Useful Tools}

When working with the power law distribution one often has to estimate partial sums of the generalized harmonic series or of even more complicated series. The following two lemmas immensely help us to do it with a reasonable precision.

\begin{lemma}\label{lem:sum-int}
    Let $f(x)$ be a non-negative integrable function on $[a, b]$, were $a, b$ are some integer numbers.
    \begin{enumerate}
        \item If $f(x)$ is non-decreasing on $[a, b]$, then we have
        \begin{align*}
            f(a) + \int_a^b f(x) dx \le \sum_{i = a}^b f(i) \le \int_a^b f(x) dx + f(b).
        \end{align*}
        \item If $f(x)$ is non-increasing on $[a, b]$, then we have
        \begin{align*}
            \int_a^b f(x) dx + f(b) \le \sum_{i = a}^b f(i) \le f(a) + \int_a^b f(x) dx.
        \end{align*}
        \item If there exists some real number $c \in (a, b)$ such that $f(x)$ is non-decreasing on $[a, c]$ and $f(x)$ is non-increasing on $[c, b]$, then we have
        \begin{align*}
            f(a) + f(b) - f(c) + \int_a^b f(x) dx \le \sum_{i = a}^b f(i) \le \int_a^b f(x) dx + f(c).
        \end{align*}
    \end{enumerate} 
\end{lemma}

\begin{proof}
    To prove all three statements of the lemma we use the following expression of the sum, which is true for all functions.
    \begin{align*}
        \sum_{i = a}^b f(i) = \int_{a}^{b + 1} f(\lfloor x \rfloor) dx.
    \end{align*}
    In the case of \textbf{non-decreasing} function $f$ for all $x \in [a, b]$ we have
    \begin{align*}
        f(x - 1) \le f(\lfloor x \rfloor) \le f(x).
    \end{align*}
    Therefore, we compute
    \begin{align*}
        \sum_{i = a}^b f(i) &= f(a) + \sum_{i = a + 1}^{b} f(i) = f(a) + \int_{a + 1}^{b + 1} f(\lfloor x \rfloor) dx \\
        & \ge f(a) + \int_{a + 1}^{b + 1} f(x - 1) dx = f(a) + \int_{a}^{b} f(x) dx.  
    \end{align*}
    We also compute the upper bound.
    \begin{align*}
        \sum_{i = a}^b f(i) &= \sum_{i = a}^{b - 1} f(i) + f(b) = \int_{a}^{b} f(\lfloor x \rfloor) dx  + f(b)\\
        & \le \int_{a}^{b} f(x) dx + f(b).  
    \end{align*}
    These two bounds are illustrated in Figure~\ref{fig:integral-increasing}.
    \begin{figure}
        \caption{Illustration of the bounding integrals for an increasing $f(x)$. In both plots the gray area is equal to the estimated sum. In the left plot the red area is an upper bound for the sum. In the right plot the blue area is the lower bound for the sum.}
        \label{fig:integral-increasing}
        \begin{center}
            \begin{tikzpicture}
                \begin{axis}[
                ylabel={$f(x)$},
                xlabel={$x$},
                legend pos=outer north east,
                ymin=0,
                xtick={1, 4, 5},
                xticklabels={$a$, $b$, $b + 1$},
                ymajorticks=false,
                width=0.5\textwidth,
                legend pos=north west]
                    \addplot[domain=1:5, samples=81, draw=red, ultra thick]{f1(x)};
                    \addlegendentry{$f(x)$};
                    \addplot[domain=1:5, samples=81, draw=none, fill=red, fill opacity=0.2]{f1(x)} \closedcycle;
                    \addplot[thick, dashed, draw=red] coordinates
                    {(1, 1) (1, 0)};
                    \addplot[thick, dashed, draw=red] coordinates
                    {(4, 8) (4, 0)};
                    \addplot[thick, dashed, draw=red] coordinates
                    {(5, 8) (5, 0)};
                
                    \addplot[ultra thick] coordinates
                    {(1,1) (2, 1)};
                    \addplot[draw=none, fill=black, fill opacity=0.2]coordinates
                    {(1,1) (2, 1)} \closedcycle;
                    \addplot[ultra thick] coordinates
                    {(2,2.83) (3, 2.83)};
                    \addplot[draw=none, fill=black, fill opacity=0.2]coordinates
                    {(2,2.83) (3, 2.83)} \closedcycle;
                    \addplot[ultra thick] coordinates
                    {(3,5.2) (4, 5.2)};
                    \addplot[draw=none, fill=black, fill opacity=0.2]coordinates
                    {(3,5.2) (4, 5.2)} \closedcycle;
                    \addplot[ultra thick] coordinates
                    {(4,8) (5, 8)};
                    \addplot[draw=none, fill=black, fill opacity=0.2]coordinates
                    {(4,8) (5, 8)} \closedcycle;
                    \node [red] at (3, 1.2) {$\int\limits_a^b f(x) dx$};
                    \node [red] at (4.5, 4) {$f(b)$}; 
                \end{axis}
            \end{tikzpicture}
            \begin{tikzpicture}
                \begin{axis}[
                ylabel={$f(x)$},
                xlabel={$x$},
                legend pos=outer north east,
                ymin=0,
                xtick={1, 2, 5},
                xticklabels={$a$, $a + 1$, $b + 1$},
                ymajorticks=false,
                width=0.5\textwidth,
                legend pos=north west]
                    \addplot[domain=1:5, samples=81, draw=blue, ultra thick]{f2(x)};
                    \addlegendentry{$f(x - 1)$};
                    \addplot[domain=1:5, samples=81, draw=none, fill=blue, fill opacity=0.2]{f2(x)} \closedcycle;
                    \addplot[thick, dashed, draw=blue] coordinates
                    {(1, 1) (1, 0)};
                    \addplot[thick, dashed, draw=blue] coordinates
                    {(2, 1) (2, 0)};
                    \addplot[thick, dashed, draw=blue] coordinates
                    {(5, 8) (5, 0)};
                
                    \addplot[ultra thick] coordinates
                    {(1,1) (2, 1)};
                    \addplot[draw=none, fill=black, fill opacity=0.2]coordinates
                    {(1,1) (2, 1)} \closedcycle;
                    \addplot[ultra thick] coordinates
                    {(2,2.83) (3, 2.83)};
                    \addplot[draw=none, fill=black, fill opacity=0.2]coordinates
                    {(2,2.83) (3, 2.83)} \closedcycle;
                    \addplot[ultra thick] coordinates
                    {(3,5.2) (4, 5.2)};
                    \addplot[draw=none, fill=black, fill opacity=0.2]coordinates
                    {(3,5.2) (4, 5.2)} \closedcycle;
                    \addplot[ultra thick] coordinates
                    {(4,8) (5, 8)};
                    \addplot[draw=none, fill=black, fill opacity=0.2]coordinates
                    {(4,8) (5, 8)} \closedcycle;
                    \node [blue] at (3.8, 1.2) {\footnotesize{$\int\limits_{a + 1}^{b + 1} f(x - 1) dx$}};
                    \node [blue] at (1.5, 0.5) {$f(a)$}; 
                \end{axis}
            \end{tikzpicture}
        \end{center}
    \end{figure}    

    In the case of \textbf{non-increasing} function $f$ for all $x \in[a,b]$ we have 
    \begin{align*}
        f(x - 1) \ge f(\lfloor x \rfloor) \ge f(x).
    \end{align*}
    We use similar arguments as for a non-decreasing function and compute
    \begin{align*}
        \sum_{i = a}^b f(i) &= \sum_{i = a}^{b - 1} f(i) + f(b) = \int_{a}^{b} f(\lfloor x \rfloor) dx  + f(b)\\
        & \ge \int_{a}^{b} f(x) dx + f(b).  
    \end{align*}
    We also compute the upper bound.
    \begin{align*}
        \sum_{i = a}^b f(i) &= f(a) + \sum_{i = a + 1}^{b} f(i) = f(a) + \int_{a + 1}^{b + 1} f(\lfloor x \rfloor) dx \\
        & \le f(a) + \int_{a + 1}^{b + 1} f(x - 1) dx = f(a) + \int_{a}^{b} f(x) dx.  
    \end{align*}
    These two bounds are illustrated in Figure~\ref{fig:integral-decreasing}.

    \begin{figure}
        \caption{Illustration of the bounding integrals for an decreasing $f(x)$. In both plots the gray area is equal to the estimated sum. In the left plot the red area is an upper bound for the sum. In the right plot the blue area is the lower bound for the sum.}
        \label{fig:integral-decreasing}
        \begin{center}
            \begin{tikzpicture}
                \begin{axis}[
                ylabel={$f(x)$},
                xlabel={$x$},
                legend pos=outer north east,
                ymin=0,
                xtick={1, 2, 5},
                xticklabels={$a$, $a + 1$, $b + 1$},
                ymajorticks=false,
                width=0.5\textwidth,
                legend pos=north east]
                    \addplot[domain=1:5, samples=81, draw=red, ultra thick]{f3(x)};
                    \addlegendentry{$f(x - 1)$};
                    \addplot[domain=1:5, samples=81, draw=none, fill=red, fill opacity=0.2]{f3(x)} \closedcycle;
                    \addplot[thick, dashed, draw=red] coordinates
                    {(1, 1) (1, 0)};
                    \addplot[thick, dashed, draw=red] coordinates
                    {(2, 1) (2, 0)};
                    \addplot[thick, dashed, draw=red] coordinates
                    {(5, 0.25) (5, 0)};
                
                    \addplot[ultra thick] coordinates
                    {(1, 1) (2, 1)};
                    \addplot[draw=none, fill=black, fill opacity=0.2]coordinates
                    {(1, 1) (2, 1)} \closedcycle;
                    \addplot[ultra thick] coordinates
                    {(2, 0.5) (3, 0.5)};
                    \addplot[draw=none, fill=black, fill opacity=0.2]coordinates
                    {(2, 0.5) (3, 0.5)} \closedcycle;
                    \addplot[ultra thick] coordinates
                    {(3, 0.33) (4, 0.33)};
                    \addplot[draw=none, fill=black, fill opacity=0.2]coordinates
                    {(3, 0.33) (4, 0.33)} \closedcycle;
                    \addplot[ultra thick] coordinates
                    {(4, 0.25) (5, 0.25)};
                    \addplot[draw=none, fill=black, fill opacity=0.2]coordinates
                    {(4, 0.25) (5, 0.25)} \closedcycle;
                    \node [red] at (3.5, 0.15) {$\int\limits_{a + 1}^{b + 1} f(x - 1) dx$};
                    \node [red] at (1.5, 0.5) {$f(a)$}; 
                \end{axis}
            \end{tikzpicture}
            \begin{tikzpicture}
                \begin{axis}[
                ylabel={$f(x)$},
                xlabel={$x$},
                legend pos=outer north east,
                ymin=0,
                xtick={1, 4, 5},
                xticklabels={$a$, $b$, $b + 1$},
                ymajorticks=false,
                width=0.5\textwidth,
                legend pos=north east]
                    \addplot[domain=1:5, samples=81, draw=blue, ultra thick]{f4(x)};
                    \addlegendentry{$f(x)$};
                    \addplot[domain=1:5, samples=81, draw=none, fill=blue, fill opacity=0.2]{f4(x)} \closedcycle;
                    \addplot[thick, dashed, draw=blue] coordinates
                    {(1, 1) (1, 0)};
                    \addplot[thick, dashed, draw=blue] coordinates
                    {(4, 0.25) (4, 0)};
                    \addplot[thick, dashed, draw=blue] coordinates
                    {(5, 0.25) (5, 0)};
                
                    \addplot[ultra thick] coordinates
                    {(1, 1) (2, 1)};
                    \addplot[draw=none, fill=black, fill opacity=0.2]coordinates
                    {(1, 1) (2, 1)} \closedcycle;
                    \addplot[ultra thick] coordinates
                    {(2, 0.5) (3, 0.5)};
                    \addplot[draw=none, fill=black, fill opacity=0.2]coordinates
                    {(2, 0.5) (3, 0.5)} \closedcycle;
                    \addplot[ultra thick] coordinates
                    {(3, 0.33) (4, 0.33)};
                    \addplot[draw=none, fill=black, fill opacity=0.2]coordinates
                    {(3, 0.33) (4, 0.33)} \closedcycle;
                    \addplot[ultra thick] coordinates
                    {(4, 0.25) (5, 0.25)};
                    \addplot[draw=none, fill=black, fill opacity=0.2]coordinates
                    {(4, 0.25) (5, 0.25)} \closedcycle;
                    \node [blue] at (2.3, 0.2) {$\int\limits_a^b f(x) dx$};
                    \node [blue] at (4.5, 0.12) {$f(b)$}; 
                \end{axis}
            \end{tikzpicture}
        \end{center}
    \end{figure}

    For a function \textbf{with maximum in} $c \in (a, b)$ we note that for all $x \in [a, c]$ we have
    \begin{align*}
        f(x - 1) \le f(\lfloor x \rfloor) \le f(x),
    \end{align*}
    for all $x \in [c + 1, b]$ (if $c + 1 \le b$) we have 
    \begin{align*}
        f(x - 1) \ge f(\lfloor x \rfloor) \ge f(x),
    \end{align*}
    and for all $x \in [c, \min\{c + 1, b\}]$ we have
    \begin{align*}
        f(c) \ge f(\lfloor x \rfloor) \ge \min\{f(x), f(x - 1)\} \ge f(x) + f(x - 1) - f(c).
    \end{align*}

    When $b \ge c + 1$, then we compute the lower and upper bounds as follows (this case is illustrated in Figure~\ref{fig:integral-with-max}).
    \begin{align*}
        \sum_{i = a}^b f(i) &= f(a) + \sum_{i = a + 1}^{b - 1} f(i)  + f(b)  = f(a) + f(b) + \int_{a + 1}^b f(\lfloor x \rfloor) dx \\
        &\ge f(a) + f(b) + \int_{a + 1}^c f(x - 1) dx \\
        &+ \int_{c}^{c + 1} \left(f(x) + f(x - 1) - f(c)\right) dx + \int_{c + 1}^b f(x) dx \\
        &= f(a) + f(b) - f(c) + \int_{a + 1}^{c + 1} f(x - 1) dx + \int_{c}^b f(x) dx\\
        &= f(a) + f(b) - f(c) + \int_{a}^{b} f(x) dx.\\
        \sum_{i = a}^b f(i) &= \int_{a}^{b + 1} f(\lfloor x \rfloor) dx \\
        &\le \int_a^c f(x) dx + \int_c^{c + 1} f(c) dx +  \int_{c + 1}^{b + 1} f(x - 1) dx \\
        &= f(c) + \int_a^b f(x) dx.
    \end{align*}

    When $b < c + 1$ the upper bound can be shown with the same arguments. For the lower bound we have
    \begin{align*}
        \sum_{i = a}^b f(i) &= f(a) + \sum_{i = a + 1}^{b - 1} f(i)  + f(b)  = f(a) + f(b) + \int_{a + 1}^b f(\lfloor x \rfloor) dx \\
        &\ge f(a) + f(b) + \int_{a + 1}^c f(x - 1) dx \\
        &+ \int_{c}^{b} \left(f(x) + f(x - 1) - f(c)\right) dx \\
        &= f(a) + f(b) + \int_{a + 1}^{b} f(x - 1) dx + \int_{c}^{b} f(x) dx - \int_c^b  f(c) dx \\
        &\ge f(a) + f(b) + \int_{a + 1}^{b} f(x - 1) dx + \int_{c}^{b} f(x) dx \\
        &+ \int_{b - 1}^c (f(x) - f(c)) dx - (b - c) f(c) \\
        &= f(a) + f(b) - f(c) + \int_{a}^{b} f(x) dx.\\
    \end{align*}

    \begin{figure}
        \caption{The illustration of the bounding integrals for a function which is increasing in interval $[a,c]$ and decreasing in interval $[c, b]$. In both plots the gray area is equal to the estimated sum. In the left plot the red area is an upper bound for the sum. The left part of the red line represents $f(x)$ in interval $[a, c]$, the central part shows the line $y = f(c)$ in interval $[c, c + 1]$ and the right part stands for $f(x - 1)$ in interval $[c + 1, b + 1]$. In the right plot the solid blue line stands for $\min\{f(x), f(x - 1)\}$ and the sum of blue areas is not a lower bound for the sum, unless we subtract the overlapped area in interval $[c, c + 1]$ together with the small areas above the solid blue line. This area is at most $f(c)$.}
        \label{fig:integral-with-max}
        \begin{center}
            \begin{tikzpicture}
                \begin{axis}[
                ylabel={$f(x)$},
                xlabel={$x$},
                ymin=0,
                xtick={1, 3.46, 4.46, 7},
                xticklabels={$a$, $c$, $c + 1$, $b + 1$},
                ymajorticks=false,
                width=0.5\textwidth]
                    \addplot[domain=1.01:7, samples=81, draw=red, ultra thick]{f5(x)};
                    \addplot[domain=1.01:7, samples=81, draw=none, fill=red, fill opacity=0.2]{f5(x)} \closedcycle;
                    \addplot[thick, dashed, draw=red] coordinates
                    {(1, 1) (1, 0)};
                    \addplot[thick, dashed, draw=red] coordinates
                    {(3.46, 6.45) (3.46, 0)};
                    \addplot[thick, dashed, draw=red] coordinates
                    {(4.46, 6.45) (4.46, 0)};
                    \addplot[thick, dashed, draw=red] coordinates
                    {(7, 4.54) (7, 0)};
                
                    \addplot[ultra thick] coordinates
                    {(1, 1) (2, 1)};
                    \addplot[draw=none, fill=black, fill opacity=0.2]coordinates
                    {(1, 1) (2, 1)} \closedcycle;
                    \addplot[ultra thick] coordinates
                    {(2, 2.82) (3, 2.82)};
                    \addplot[draw=none, fill=black, fill opacity=0.2]coordinates
                    {(2, 2.82) (3, 2.82)} \closedcycle;
                    \addplot[ultra thick] coordinates
                    {(3, 5.19) (4, 5.19)};
                    \addplot[draw=none, fill=black, fill opacity=0.2]coordinates
                    {(3, 5.19) (4, 5.19)} \closedcycle;
                    \addplot[ultra thick] coordinates
                    {(4, 6) (5, 6)};
                    \addplot[draw=none, fill=black, fill opacity=0.2]coordinates
                    {(4, 6) (5, 6)} \closedcycle;
                    \addplot[ultra thick] coordinates
                    {(5, 5.37) (6, 5.37)};
                    \addplot[draw=none, fill=black, fill opacity=0.2]coordinates
                    {(5, 5.37) (6, 5.37)} \closedcycle;
                    \addplot[ultra thick] coordinates
                    {(6, 4.9) (7, 4.9)};
                    \addplot[draw=none, fill=black, fill opacity=0.2]coordinates
                    {(6, 4.9) (7, 4.9)} \closedcycle;
                    \node [red, rotate=90] at (2.3, 1.4) {$\int\limits_{a}^{b} f(x) dx$};
                    \node [red, rotate = 90] at (4, 3) {$f(c)$}; 
                    \node [red, rotate = 90] at (5.8, 2.5) {$\int\limits_{c + 1}^{b + 1} f(x - 1) dx$};
                \end{axis}
            \end{tikzpicture}
            \begin{tikzpicture}
                \begin{axis}[
                ylabel={$f(x)$},
                xlabel={$x$},
                ymin=0,
                xtick={1, 2, 3.46, 4.46, 6, 7},
                xticklabels={$a$, $a + 1$, $c$, $c + 1$, $b$, $b + 1$},
                ymajorticks=false,
                width=0.5\textwidth]
                    \addplot[domain=2:6, samples=81, draw=blue, ultra thick]{f6(x)};
                    \addplot[domain=3.46:4.46, samples=21, draw=blue, dashed, ultra thick]{f7(x)};
                    \addplot[domain=2:6, samples=81, draw=none, fill=blue, fill opacity=0.2]{f6(x)} \closedcycle;
                    \addplot[domain=3.46:4.46, samples=21, draw=none, fill=blue, fill opacity=0.2]{f7(x)} \closedcycle;
                    \addplot[draw=none, fill=blue, fill opacity=0.2]coordinates
                    {(1, 1) (2, 1)} \closedcycle;
                    \addplot[draw=none, fill=blue, fill opacity=0.2]coordinates
                    {(6, 4.9) (7, 4.9)} \closedcycle;
                    \addplot[thick, dashed, draw=blue] coordinates
                    {(1, 1) (1, 0)};
                    \addplot[thick, dashed, draw=blue] coordinates
                    {(3.46, 6.45) (3.46, 0)};
                    \addplot[thick, dashed, draw=blue] coordinates
                    {(4.46, 6.45) (4.46, 0)};
                    \addplot[thick, dashed, draw=blue] coordinates
                    {(7, 4.54) (7, 0)};
                    \addplot[thick, dashed, draw=blue] coordinates
                    {(2, 1) (2, 0)};
                    \addplot[thick, dashed, draw=blue] coordinates
                    {(6, 4.9) (6, 0)};
                
                    \addplot[ultra thick] coordinates
                    {(1, 1) (2, 1)};
                    \addplot[draw=none, fill=black, fill opacity=0.2]coordinates
                    {(1, 1) (2, 1)} \closedcycle;
                    \addplot[ultra thick] coordinates
                    {(2, 2.82) (3, 2.82)};
                    \addplot[draw=none, fill=black, fill opacity=0.2]coordinates
                    {(2, 2.82) (3, 2.82)} \closedcycle;
                    \addplot[ultra thick] coordinates
                    {(3, 5.19) (4, 5.19)};
                    \addplot[draw=none, fill=black, fill opacity=0.2]coordinates
                    {(3, 5.19) (4, 5.19)} \closedcycle;
                    \addplot[ultra thick] coordinates
                    {(4, 6) (5, 6)};
                    \addplot[draw=none, fill=black, fill opacity=0.2]coordinates
                    {(4, 6) (5, 6)} \closedcycle;
                    \addplot[ultra thick] coordinates
                    {(5, 5.37) (6, 5.37)};
                    \addplot[draw=none, fill=black, fill opacity=0.2]coordinates
                    {(5, 5.37) (6, 5.37)} \closedcycle;
                    \addplot[ultra thick] coordinates
                    {(6, 4.9) (7, 4.9)};
                    \addplot[draw=none, fill=black, fill opacity=0.2]coordinates
                    {(6, 4.9) (7, 4.9)} \closedcycle;

                    \node [blue, rotate=70] at (3.5, 2) {\footnotesize{$\int\limits_{a+1}^{b+1} f(x - 1) dx$}};
                    \node [blue, rotate = 70] at (5, 2.5) {\footnotesize{$\int\limits_{c}^{b} f(x) dx$}}; 
                    \node [blue] at (1.5, 0.5) {\footnotesize{$f(a)$}};
                    \node [blue] at (6.5, 2.5) {\footnotesize{$f(b)$}}; 
                \end{axis}
            \end{tikzpicture}
        \end{center}
    \end{figure}

\end{proof}

\begin{lemma}\label{lem:part-sum}
    Let $a$ and $b$ be some positive integers such that $a \le b$. Then for all $\beta \ne 1$ we have
    \begin{align*}
        \sum_{i = a}^b i^{-\beta} = \frac{b^{1 - \beta} - a^{1 - \beta}}{1 - \beta} + \delta,
    \end{align*}
    where $\delta$ satisfies
    \begin{align*}
        \min\{a^{-\beta}, b^{-\beta}\} \le \delta \le \max\{a^{-\beta}, b^{-\beta}\}.
    \end{align*}
    If $\beta = 1$, then we have 
    \begin{align*}
        \sum_{i = a}^b i^{-\beta} = \ln\left(\frac{b}{a}\right) + \delta,
    \end{align*}
    with $\delta \in [\frac{1}{b}, \frac{1}{a}]$.
\end{lemma}

\begin{proof}
    For $\beta < 0$ function $f(x) = x^{-\beta}$ is a increasing function and for $\beta \ge 0$ this is a non-increasing function. Hence, we can use the first two statements of Lemma~\ref{lem:sum-int} to estimate the sum. For $\beta \ne 1$ we have
    \begin{align*}
        \int_a^b x^{-\beta} dx = \frac{b^{1 - \beta} - a^{1 - \beta}}{1 - \beta}.
    \end{align*}
    Therefore, for $\beta < 0$ we have
    \begin{align*}
        a^{-\beta} + \frac{b^{1 - \beta} - a^{1 - \beta}}{1 - \beta} \le \sum_{i = a}^b i^{-\beta} \le \frac{b^{1 - \beta} - a^{1 - \beta}}{1 - \beta} + b^{-\beta}.
    \end{align*}
    For $\beta \ge 0$ (except for $\beta = 1$) we have
    \begin{align*}
        \frac{b^{1 - \beta} - a^{1 - \beta}}{1 - \beta} + b^{-\beta} \le \sum_{i = a}^b i^{-\beta} \le a^{-\beta} + \frac{b^{1 - \beta} - a^{1 - \beta}}{1 - \beta}.
    \end{align*}
    For $\beta = 1$ we have
    \begin{align*}
        \int_a^b x^{-\beta} dx = \ln(b) - \ln(a) = \ln\left(\frac{b}{a}\right),
    \end{align*}
    hence,
    \begin{align*}
        \ln\left(\frac{b}{a}\right) + \frac{1}{b} \le \sum_{i = a}^b i^{-\beta} \le \frac{1}{a} + \ln\left(\frac{b}{a}\right)
    \end{align*}


    We complete the proof by noting that 
    \begin{align*}
        \max\{a^{-\beta}, b^{-\beta}\} = \begin{cases}
            a^{-\beta}, \text{ if } \beta \ge 0, \\
            b^{-\beta}, \text{ if } \beta < 0,
        \end{cases}
    \end{align*}
    and 
    \begin{align*}
        \min\{a^{-\beta}, b^{-\beta}\} = \begin{cases}
            b^{-\beta}, \text{ if } \beta \ge 0, \\
            a^{-\beta}, \text{ if } \beta < 0,
        \end{cases}
    \end{align*}
\end{proof}

The following lemma will help us to simplify the estimates given by Lemmas~\ref{lem:sum-int} and~\ref{lem:part-sum}.

\begin{lemma}\label{lem:change-power}
    For all $x \ge 1$, all $y > x$ and all $\beta \ge 0$ such that $\beta \ne 1$ we have
    \begin{align*}
        x^{-\beta} - y^{-\beta} \le \frac{\beta}{\beta - 1}(x^{1 - \beta} - y^{1 - \beta}).
    \end{align*}
\end{lemma}
\begin{proof}
    For $\beta = 0$ both sides of the inequality are equal to zero. For $\beta > 0$ we have
    \begin{align*}
        x^{-\beta} - y^{-\beta} = \beta \int_x^y t^{-1-\beta} dt \le \beta \int_x^y t^{-\beta} dt = \frac{\beta}{\beta - 1}(x^{1 - \beta} - y^{1 - \beta}).
    \end{align*}
\end{proof}

We also use Wald's equation~\cite{Wald45} in our analysis. In particular, it helps us estimate the expected number of the fitness evaluation which the algorithm performs when we know the expected number of iterations and the expected number of fitness evaluations per iteration.

\begin{lemma}[Wald's equation]\label{lem:wald}
	Let $(X_t)_{t \in \N}$ be a sequence of real-valued random variables and let $T$ be a positive integer random variable. Let also all following conditions be true.
	\begin{enumerate}
		\item All $X_t$ have the same finite expectation.
		\item For all $t \in \N$ we have $E[X_t \mathds{1}_{\{T \ge t\}}] = E[X_t] \Pr[T \ge t]$.
		\item $\sum_{t = 1}^{+\infty} E[|X_t| \mathds{1}_{\{T \ge t\}}] < \infty$.
		\item $E[T]$ is finite.
	\end{enumerate}
	Then we have
	\[
		E\left[\sum_{t = 1}^{T} X_t\right] = E[T]E[X_1].	
	\]
\end{lemma}

In our black-box complexity analysis, we use the following combinatorial result~\cite[p.41]{koepf-vandermonde}:

\begin{lemma}[Chu-Vandermonde identity]\label{lem:vandermonde}
The following holds for all $m \le n$ and all $k \le \min\{m, n-m\}$.
\begin{equation*}
\sum_{j=0}^{k} \binom{m}{j} \binom{n-m}{k-j} = \binom{n}{k}.
\end{equation*}
\end{lemma}

In particular, we are interested in the following corollary.
\begin{lemma}\label{lem:vm-corollary}
The following holds for all $m \le n$ and all $x \le \min\{m, n-m\}$.
\begin{equation*}
    \binom{m}{x} \binom{n-m}{x} \le \binom{n}{m}.
\end{equation*}
\end{lemma}
\begin{proof}
    Follows from Lemma~\ref{lem:vandermonde} by considering $k=m$, using $\binom{m}{x} = \binom{m}{m-x}$
    and noting that the left-hand side is just one of the summands.
\end{proof}

We also use the following classic result for the lower bounds.

\begin{lemma}[\cite{DrosteJW06}, Theorem~2]\label{lem:jansen}
Let $S$ be the search space of an optimization problem. If for each $s \in S$ there is an instance
such that $s$ is the unique optimum and if each query has at most $k \ge 2$ possible answers,
then the black-box complexity is bounded below by $\lceil \log_k |S|\rceil - 1$.
\end{lemma}

Note that the proof of this result allows only a part $S$ of the (original) search space $S^{(0)}$
to be used, provided that the optimum is contained in $S$. Indeed, for the lower bound the queries
outside $S$ may be safely ignored, even if such queries are required for the problem to be solved.

\section{Runtime Analysis}
\label{sec:runtimes}

In this section we conduct a rigorous runtime analysis for the different variants of the \ollga and prove upper bounds on their runtime when they start in distance $D$ from the optimum. We start with the standard algorithm with static parameters.

\begin{theorem}
    \label{thm:fixed}
    The expected runtime of the \ollga with static parameter~$\lambda$ (and mutation rate $p = \frac \lambda n$ and crossover bias $c = \frac 1\lambda$ as recommended in~\cite{DoerrDE15}) on \onemax with initialization in distance $D$ from the optimum is 
    \[
        E[T_F] = O\left(\frac{n}{\lambda} \ln\left(\frac{n}{\lambda^2}\right) + D\lambda \right)    
    \]
    fitness evaluations. This is minimized by $\lambda = \sqrt{\frac{n\ln(D)}{D}}$, which gives a runtime guarantee of $E[T_F] = O(\sqrt{nD\ln(D)}\,)$.
\end{theorem}
\begin{proof}
    By Lemma~\ref{lem:progress} the probability $P_d$ to have progress in one iteration is $\Omega(\min\{1, \frac{d\lambda^2}{n}\})$. Therefore, the expected number of iterations until the \ollga finds the optimum is
    \begin{align*}
        E[T_I] &\le \sum_{d = 1}^D P_d^{-1} = O\left(\sum_{d = 1}^{n/\lambda^2} \frac{n}{d\lambda^2} + \sum_{d = n/\lambda^2 + 1}^{D} 1 \right) \\
        &= O\left(\frac{n}{\lambda^2} \ln\left(\frac{n}{\lambda^2}\right) + D \right).
    \end{align*}

    Since in each iteration the \ollga performs $2\lambda$ fitness evaluations ($\lambda$ in the mutation phase and $\lambda$ in the crossover phase), the expected number of fitness evaluations is by a factor of $2\lambda$ larger:
    \begin{align}
        \label{eq:static_runtime}
        E[T_F] = O\left(\frac{n}{\lambda} \ln\left(\frac{n}{\lambda^2}\right) + D\lambda \right).
    \end{align}

    The term $\frac{n}{\lambda} \ln(\frac{n}{\lambda^2})$ is decreasing in $\lambda$ and the term $D\lambda$ is increasing. Therefore, the runtime is minimized when these two terms are asymptotically equal, that is, when $\lambda = \sqrt{\frac{n\ln(D)}{D}}$. By putting this value for $\lambda$ into~\eqref{eq:static_runtime} we obtain $E[T_F] = O(\sqrt{nD\ln(D)}\,)$.\qed
\end{proof}

We move on to the \ollga with optimal fitness-dependent parameters. 

\begin{theorem}
    \label{thm:fitness-dependent}
    The expected runtime of the \ollga with fitness-dependent $\lambda = \lambda(d) = \sqrt{\frac{n}{d}}$ on \onemax with initialization in distance $D$ from the optimum is $E[T_F] = O(\sqrt{nD}).$
\end{theorem}
\begin{proof}
By Lemma~\ref{lem:progress}, the probability of having a true progress in one iteration is constant. The cost of one iteration, however, is a non-constant $2\lambda$. Thus the expected runtime for a fitness increase is $O(\lambda)$. Pessimistically assuming that we do not increase the fitness by more than one, we obtain 
\begin{align*}
    E[T_F] = O\left(\sum_{d = 1}^{D} \sqrt{\frac{n}{d}}\right) = O(\sqrt{nD}).
\end{align*}
\qed
\end{proof}

The one-fifth rule was shown to be to keep the value of $\lambda$ close to its optimal fitness-dependent value, when starting in the random bit string. The algorithm is initialized with $\lambda =2$, which is close-to-optimal when starting in a random bit string. In the following theorem we show that even when we start in a smaller distance $D$, the one-fifth rule is capable to quickly increase $\lambda$ to its optimal value and keep it there.

\begin{theorem}
    \label{thm:adaptive}
    The expected runtime of the \ollga with self-adjusting $\lambda$ (according to the one-fifth rule) on \onemax with initialization in distance $D$ from the optimum is $E[T_F] = O(\sqrt{nD}).$
\end{theorem}
\begin{proof}
    Let $d' \in [0..D]$ be the first distance at which the algorithm reached $\lambda \ge \sqrt{\frac{n}{d'}}$. At each distance $d > d'$ the algorithm did not manage to do so, hence it spent at most
    \begin{align*}
        T_F(d) = \sum_{i = 1}^{\log_A(\sqrt{\frac{n}{d}})} A^i = O\left(\sqrt{\frac{n}{d}}\right) 
    \end{align*}  
    fitness evaluations at that distance. Repeating the arguments of Lemma~16 in~\cite{DoerrD18} we can show that the starting from distance $d'$, $\lambda$ will always be close to $\sqrt{\frac{n}{d}}$, and the probability to increase the fitness is $\Theta(1)$. Therefore, at each distance we spend at most $O(\sqrt{\frac{n}{d}})$ fitness evaluations, which by analogy with Theorem~\ref{thm:fitness-dependent} yields $E[T_F] = O(\sqrt{nD})$. \qed
\end{proof}

For the fast \ollga different parameters of the power-law distribution yield different runtimes. This is shown in Theorem~\ref{thm:fast}.

\begin{theorem}
    \label{thm:fast}
    Consider the fast \ollga with parameter $u < \sqrt{n}$.  
    The expected runtime of the fast \ollga on \onemax with initialization in distance $D$ from the optimum is as shown in Table~\ref{tbl:runtimes}.
\end{theorem}

\afterpage{
    \clearpage
    \begin{landscape}
        \begin{table}[htbp]
            \begin{center}
                \begin{tabular}{|m{1.5cm}|l|l|l|}
                    \hline
                    & & & \\[-5pt] 
                    $\beta$ & 
                    Runtime if $u < \sqrt{\frac{n}{D}}$ &
                    Runtime if $u \in[\sqrt{\frac{n}{D}}, \sqrt{n}]$&
                    Runtime if $u \ge \sqrt{n}$ \\[5pt] \hline
                    & & & \\[-5pt] 

                    $\le 1$ &
                    \multirow{3}{*}{$\frac{2(3 - \beta)^2}{C(2 - \beta)} \cdot \frac{n}{u}(\ln(D) + 1)$} &
                    $\frac{2(3 - \beta)^2}{C(2 - \beta)} \cdot \left(\frac{n}{u} \ln\left(\frac{n}{u^2}\right) + Du\right)$ &
                    $\frac{2(3 - \beta)^2}{C(2 - \beta)} \cdot Du$ \\[5pt] \cline{1-1}\cline{3-4}
                    & & & \\[-5pt] 

                    $(1, 2)$ &
                    &
                    $\frac{8(3 - \beta)}{C(2 - \beta)} \cdot \frac{n}{u} \left(\ln\left(\frac{n}{u^2}\right) + \frac{2}{3 - \beta}\left(u\sqrt{\frac{D}{n}}\right)^{3 - \beta}\right)$ &
                    $\frac{8}{C(2 - \beta)} \cdot \frac{n}{u}\left(u\sqrt{\frac{D}{n}}\right)^{3 - \beta}$ \\[5pt] \hline
                    & & & \\[-5pt] 

                    $2$ &
                    $\frac{2}{C} \cdot \frac{n(\ln(u) + 1)}{u} (\ln(D) + 1)$ &
                    $\frac{8}{C} (\ln(u) + 1)\left(\frac{n}{u} \ln\left(\frac{n}{u^2}\right) +  2\sqrt{nD}\right)$ &
                    $\frac{8}{C} (\ln(u) + 1)\sqrt{nD}$ \\[5pt] \hline
                    & & & \\[-5pt] 

                    $(2, 3)$ &
                    $\frac{2(\beta - 1)}{C(\beta -2)} \cdot \frac{n}{u^{3 - \beta}} (\ln(D) + 1)$ &
                    $\frac{8(\beta - 1)}{C(\beta - 2)} \cdot\frac{n}{u^{3 - \beta}} \left(\ln\left(\frac{n}{u^2}\right) + \frac{2}{3 - \beta}\left(u\sqrt{\frac{D}{n}}\right)^{3 - \beta}\right)$ &
                    $\frac{8(\beta - 1)}{C(\beta - 2)(3 - \beta)} \cdot\sqrt{n}^{\beta - 1}\sqrt{D}^{3 - \beta}$ \\[5pt] \hline
                    & & & \\[-5pt] 

                    $3$ and $D \le \frac{n}{e^2}$ &
                    \multirow{3}{*}{$\frac{2(\beta - 1)}{C(\beta -2)} \cdot \frac{n}{\ln(u + 1)} (\ln(D) + 1)$} &
                    $\frac{2(\beta - 1)}{C(\beta - 2)} \cdot \left(\frac{n}{\ln(u + 1)} \left(\ln\left(\frac{n}{u^2}\right) + 1\right) + 2n\ln\left(\frac{\ln(u)}{\ln(\sqrt{n/D})}\right)\right)$ &
                    $\frac{2(\beta - 1)}{C(\beta - 2)} \cdot\left( \frac{n}{\ln(\sqrt{n} + 1)} + 2n\ln\left(\frac{\ln(n)}{\ln(n/D)}\right) \right)$ \\[5pt] \cline{1-1}\cline{3-4}
                    & & & \\[-5pt] 

                    $3$ and $D > \frac{n}{e^2}$ &
                    &
                    $\frac{2(\beta - 1)}{C(\beta - 2)} \cdot \left(\frac{n}{\ln(u + 1)} \left(\ln\left(\frac{n}{u^2}\right) + 1\right) + 2n\ln\ln^+(u) + \frac{2n}{\ln(2)}\right)$ &
                    $\frac{6(\beta - 1)}{C(\beta - 2)} \cdot n(\ln\ln(n) + 1)$ \\[5pt] \hline
                    & \multicolumn{3}{|c|}{} \\[-5pt] 

                    $>3$ &
                    \multicolumn{3}{|c|}{$\frac{2(\beta - 1)}{C(\beta -2)} \cdot n (\ln(D) + 1)$}
                    \\[5pt] \hline
                \end{tabular}
            \end{center}
        \caption{The expected runtime of the fast \ollga with different parameters $\beta$ and $u$ on \onemax when it is initialized in distance $D$ from the global optimum. $C$ stands for the constant hidden in the $\Omega$ notation of Lemma~\ref{lem:progress}.}
        \label{tbl:runtimes}
        \label{tbl:runtime} 
        \end{table} 
    \end{landscape}
    \clearpage
}

To prove Theorem~\ref{thm:fast} we first introduce auxiliary lemmas to estimate the probability to have a progress in one iteration depending on the current distance to the optimum and on the parameters of the power-law distribution and to estimate the expected cost of one iteration (also depending on the parameter of the power-law distribution). We start with the following lower bound on the probability of progress.

\begin{lemma}
Let $P_d$ \label{lem:pd}
be the probability that the fast \ollga improves fitness in one iteration when its current individual $x$ is in distance $d$ from the optimum. Then $P_d$ is at least as shown in Table~\ref{tbl:pd}.
\end{lemma}

\begin{table}
    \caption{Lower bounds on $P_d$ for different values of $\beta$ and $u$. In this table $C$ stands for a constant hidden in the $\Omega$ notation of Lemma~\ref{lem:progress} and $C_{\beta, u}$ is the normalization coefficient for the power-law distribution.}
	\label{tbl:pd}
	\begin{center}
		\begin{tabular}{|c||c|c|}
			\hline
            & & \\[-10pt]
            $\beta$ & 
            $P_d$ with $u \le \sqrt{\frac{n}{d}}$ & 
            $P_d$ with $u > \sqrt{\frac{n}{d}}$ \\[5pt] \hline
            & & \\[-10pt]
            $\le 1$ & 
            $\frac{CC_{\beta, u}}{3 - \beta} \cdot \frac{du^{3 - \beta}}{n}$ & 
            $\frac{CC_{\beta, u}}{3- \beta} \cdot u^{1 - \beta}$ \\[5pt] \hline
            & & \\[-10pt]
            $(1, 3)$ & 
            $\frac{CC_{\beta, u}}{2} \cdot \frac{du^{3 - \beta}}{n}$ & 
            $\frac{CC_{\beta, u}}{2} \cdot \sqrt{\frac{n}{d}}^{1 - \beta}$ \\[5pt] \hline
            & & \\[-10pt]
            $3$ & 
            $CC_{\beta, u} \cdot \frac{d\ln(u + 1)}{n}$ &
            $CC_{\beta, u} \cdot \frac{d}{n}\ln(\sqrt{\frac{n}{d}} + 1)$  \\[5pt] \hline
            & \multicolumn{2}{|c|}{} \\[-10pt]
            $>3$ & \multicolumn{2}{|c|}{$CC_{\beta, u} \cdot \frac{d}{n}$} \\[5pt] \hline
        \end{tabular}
	\end{center}
\end{table}

\begin{proof}
    Let $P_d(\lambda)$ be the probability to improve the fitness in one iteration conditional on that the distance between the current individual $x$ and the optimal bit string is $d$ and that the population size is fixed and equals to $\lambda$. By Lemma~\ref{lem:progress} this probability is $\Omega(\min\{1, \frac{d\lambda^2}{n}\})$. This implies that there exists a constant $C > 0$ such that for all $d \in [1..n]$ and all $\lambda \in [1..n]$ we have 
    \begin{align*}
        P_d(\lambda) \ge C \min\left\{1, \frac{d\lambda^2}{n}\right\}.
    \end{align*}
    Since the fast \ollga chooses population size $\lambda$ from the power-law distribution $\pow(\beta, u)$, by the law of total probability we compute
    \begin{align*}
        P_d &= \sum_{i = 1}^u \Pr[\lambda = i] P_d(i) \\
        &= \sum_{i = 1}^u C_{\beta, u} i^{-\beta} P_d(i) \\
        &\ge \sum_{i = 1}^u C_{\beta, u} i^{-\beta} C \min\left\{1, \frac{d i^2}{n}\right\}. \\
    \end{align*}
    We now consider two cases.
    
    \textbf{When $u \le \sqrt{\frac{n}{d}}$}, then we have
    \begin{align}\label{eq:pd-u-small}
        P_d &\ge C_{\beta, u} C \frac{d}{n} \sum_{i = 1}^u  i^{2 - \beta}. 
    \end{align}
    By Lemma~\ref{lem:part-sum} we estimate
    \begin{align*}
        \sum_{i = 1}^u i^{2 - \beta} \ge \begin{cases}
            \frac{u^{3 - \beta} - 1}{3 - \beta} + \min\{1, u^{2 - \beta}\}, &\text{ if } \beta \ne 3, \\
            \ln(u) + \frac{1}{u}, &\text{ if } \beta = 3.
        \end{cases}
    \end{align*}
    We now put these two estimates into~\eqref{eq:pd-u-small} for different values of $\beta$.
    
    \textbf{If} $\beta < 2$ , then we have
    \begin{align*}
        \sum_{i = 1}^u i^{2 - \beta} \ge \frac{u^{3 - \beta} - 1}{3 - \beta} + 1 = \frac{u^{3 -\beta} + (2 - \beta)}{3 - \beta} \ge \frac{u^{3 - \beta}}{3 - \beta}.
    \end{align*}
    Hence, by~\eqref{eq:pd-u-small} we obtain
    \begin{align*}
        P_d &\ge C C_{\beta, u} \cdot \frac{d}{n} \cdot \frac{u^{3 - \beta}}{(3 - \beta)}. 
    \end{align*}
    \textbf{For} $\beta \in (1, 2)$ the following weaker (by a constant factor of $\frac{2}{3 - \beta} \le 2$) lower bound also holds. This bound turns to be more convenient in our further analysis.
    \begin{align*}
        P_d &\ge C C_{\beta, u} \cdot \frac{d}{n} \cdot \frac{u^{3 - \beta}}{4}. 
    \end{align*}
    \textbf{If} $\beta \in [2, 3)$, then we have 
    \begin{align*}
        P_d &\ge CC_{\beta, u} \cdot \frac{d}{n} \cdot \left(\frac{u^{3 - \beta}  - 1}{3 - \beta} + u^{2 - \beta}\right) \\
        &= CC_{\beta, u} \cdot \frac{d}{n} \cdot \frac{u^{3 - \beta}}{(3 - \beta)} \left(1 - \frac{1}{u^{3 - \beta}} + \frac{3 - \beta}{u}\right).
    \end{align*}
    We note that $(1 - \frac{1}{u^{3 - \beta}} + \frac{3 - \beta}{u})$ is an increasing function of $u$ for $u \ge 1$, which can be shown through considering its derivative as follows.
    \begin{align*}
        \left(1 - \frac{1}{u^{3 - \beta}} + \frac{3 - \beta}{u}\right)' &= (3 - \beta) \frac{1}{u^{4 - \beta}} - (3 - \beta)\frac{1}{u^2} = (3 - \beta)\frac{(1 - u^{2 - \beta})}{u^{4 - \beta}} \ge 0.
    \end{align*}
    Hence, we have
    \begin{align*}
        \left(1 - \frac{1}{u^{3 - \beta}} + \frac{3 - \beta}{u}\right) \ge \left(1 - \frac{1}{1^{3 - \beta}} + \frac{3 - \beta}{1}\right) = 3 - \beta.
    \end{align*}
    Therefore, we compute
    \begin{align*}
        P_d &\ge CC_{\beta, u} \cdot \frac{d}{n} \cdot \frac{u^{3 - \beta}}{(3 - \beta)} (3 - \beta) = CC_{\beta, u} \cdot \frac{du^{3 - \beta}}{n} \ge  \frac{CC_{\beta, u}}{2} \cdot \frac{du^{3 - \beta}}{n},
    \end{align*}
    where in the last inequality we decreased the bound by a factor of $2$ in order to simplify our further calculations.

    \textbf{If} $\beta = 3$, then by a well-known inequality $x \ge \ln(1 + x)$, which holds for all $x > -1$, we have
    \begin{align*}
        P_d &\ge CC_{\beta, u} \cdot \frac{d}{n} \cdot \left(\ln(u) + \frac{1}{u}\right) 
        \ge CC_{\beta, u} \cdot \frac{d}{n} \cdot \left(\ln(u) + \ln\left(1 + \frac{1}{u}\right)\right)\\
        &= CC_{\beta, u} \cdot \frac{d\ln(u + 1)}{n}.
    \end{align*}
    \textbf{If} $\beta > 3$, then we have
    \begin{align*}
        P_d &\ge CC_{\beta, u} \cdot \frac{d}{n} \sum_{i = 1}^u i^{2 - \beta} \ge CC_{\beta, u} \cdot \frac{d}{n} \sum_{i = 1}^1 i^{2 - \beta} = CC_{\beta, u} \cdot \frac{d}{n}.
    \end{align*}

    \textbf{When $u > \sqrt{\frac{n}{d}}$}, then we have
    \begin{align}\label{eq:pd-u-large}
        P_d &\ge C_{\beta, u} C \sum_{i = 1}^u  f(i),
    \end{align}
    where $f$ is a function defined on $[1,u]$ as
    \begin{align*}
        f(x) = \min\left\{x^{-\beta}, \frac{dx^{2 - \beta}}{n}\right\} = x^{-\beta} \min\left\{1, \frac{dx^2}{n}\right\}.
    \end{align*}
    Note that $f$ has the following properties depending on parameter $\beta$. It is an increasing function if $\beta < 0$, it is increasing in $[1, \sqrt{\frac{n}{d}}]$ and non-increasing in $[\sqrt{\frac{n}{d}}, u]$, if $\beta \in [0, 2)$, and it is non-increasing, if $\beta \ge 2$. 
    
    Therefore, \textbf{if} $\beta < 0$, then by the first statement of Lemma~\ref{lem:sum-int} we have
    \begin{align*}
        \sum_{i = 1}^u f(i) &\ge f(1) + \int_1^u f(x) dx = \frac{d}{n} + \int_{1}^{\sqrt{\frac{n}{d}}} \frac{d}{n} x^{2 - \beta} dx + \int_{\sqrt{\frac{n}{d}}}^u x^{-\beta} dx \\
        &= \frac{d}{n} + \frac{d}{n} \cdot \frac{\sqrt{\frac{n}{d}}^{3 - \beta} - 1}{3 - \beta} + \frac{u^{1 - \beta} - \sqrt{\frac{n}{d}}^{1 - \beta}}{1 - \beta} \\
        &= \frac{d}{n} \left(1 - \frac{1}{3 - \beta}\right) - \sqrt{\frac{n}{d}}^{1 - \beta} \left(\frac{1}{1 - \beta} - \frac{1}{3 - \beta}\right) + \frac{u^{1 - \beta}}{1 - \beta}\\
        &\ge u^{1 - \beta} \left(\frac{1}{3 - \beta} - \frac{1}{1 - \beta}\right) + \frac{u^{1 - \beta}}{1 - \beta} = \frac{u^{1 - \beta}}{3 - \beta}
    \end{align*} 
    where in the last inequality we used $u \ge \sqrt{\frac{n}{d}}$ and $\frac{d}{n} (1 - \frac{1}{3 - \beta}) > 0$. By~\eqref{eq:pd-u-large}, we have
    \begin{align*}
        P_d \ge CC_{\beta, u} \cdot \frac{u^{1 - \beta}}{3 - \beta}.
    \end{align*}

    \textbf{If} $\beta \in [0, 1)$, then by the third statement of Lemma~\ref{lem:sum-int} we have
    \begin{align}\label{eq:beta-02-1}
        \begin{split}
            \sum_{i = 1}^u f(i) &\ge f(1) + f(u) - f\left(\sqrt{\frac{n}{d}}\right) + \int_1^u f(x) dx \\
            &= \frac{d}{n} + u^{-\beta} - \sqrt{\frac{n}{d}}^{- \beta} + \int_1^{\sqrt{\frac{n}{d}}} \frac{d}{n} x^{2 - \beta} dx + \int_{\sqrt{\frac{n}{d}}}^u x^{-\beta} dx \\
            &= \frac{d}{n} + u^{-\beta} - \sqrt{\frac{n}{d}}^{- \beta} + \frac{d}{n} \cdot \frac{\sqrt{\frac{n}{d}}^{3 - \beta} - 1}{3 - \beta} + \frac{u^{1 - \beta} - \sqrt{\frac{n}{d}}^{1 - \beta}}{1 - \beta} \\
            &= \frac{d}{n} \left(1 - \frac{1}{3 - \beta}\right) - \left(\left(\sqrt{\frac{n}{d}}\right)^{-\beta} - u^{-\beta}\right) \\
            &+ \left(u^{1 - \beta} - \left(\sqrt{\frac{n}{d}}\right)^{1 - \beta}\right)\left(\frac{1}{1 - \beta} - \frac{1}{3 - \beta}\right) + \frac{u^{1 - \beta}}{3 - \beta}.
        \end{split}
    \end{align}
    By Lemma~\ref{lem:change-power} we have $(\sqrt{\frac{n}{d}})^{-\beta} - u^{-\beta} \le \frac{\beta}{\beta - 1}( (\sqrt{\frac{n}{d}})^{1-\beta} - u^{1-\beta})$, therefore we have
    \begin{align}\label{eq:beta-02-2}
        \begin{split}
            \sum_{i = 1}^u f(i) &\ge \frac{d}{n} \cdot \frac{(2 - \beta)}{(3 - \beta)} - \frac{\beta}{\beta - 1}\left(\left(\sqrt{\frac{n}{d}}\right)^{1-\beta} - u^{1-\beta}\right) \\
            &+ \left(u^{1 - \beta} - \left(\sqrt{\frac{n}{d}}\right)^{1 - \beta}\right)\left(\frac{1}{1 - \beta} - \frac{1}{3 - \beta}\right) + \frac{u^{1 - \beta}}{3 - \beta} \\
            &= \frac{d}{n} \cdot \frac{(2 - \beta)}{(3 - \beta)} + \frac{u^{1 - \beta}}{3 - \beta}  \\
            &+ \left(u^{1 - \beta} - \left(\sqrt{\frac{n}{d}}\right)^{1 - \beta}\right)\left(\frac{1}{1 - \beta} - \frac{1}{3 - \beta} - \frac{\beta}{1 - \beta}\right)\\
            &= \frac{d}{n} \cdot \frac{(2 - \beta)}{(3 - \beta)} + \frac{(2 - \beta)}{(3 - \beta)} \cdot \left(u^{1 - \beta} - \left(\sqrt{\frac{n}{d}}\right)^{1 - \beta}\right) + \frac{u^{1 - \beta}}{3 - \beta} \\
            &\ge \frac{u^{1 - \beta}}{3 - \beta},
        \end{split}
    \end{align}
    since $\frac{2 - \beta}{3 - \beta} > 0$ and $u^{1 - \beta} - (\sqrt{\frac{n}{d}})^{1 - \beta} > 0$. When we put this into~\eqref{eq:pd-u-large}, we obtain
    \begin{align*}
        P_d \ge CC_{\beta, u} \cdot \frac{u^{1 - \beta}}{3 - \beta}.
    \end{align*}

    \textbf{If} $\beta = 1$, we use a different way to estimate the sum and obtain
    \begin{align*}
        \sum_{i = 1}^u f(i) &\ge \sum_{i = 1}^{\lfloor \sqrt{\frac{n}{d}} \rfloor} \frac{id}{n} + \frac{1}{\lfloor\sqrt{\frac{n}{d}} \rfloor + 1} = \frac{d}{n} \cdot \frac{\lfloor \sqrt{\frac{n}{d}} \rfloor (\lfloor \sqrt{\frac{n}{d}} \rfloor + 1)}{2} + \frac{1}{\lfloor\sqrt{\frac{n}{d}} \rfloor + 1}\\
        &\ge \frac{d}{n} \cdot \frac{\sqrt{\frac{n}{d}} \left(\sqrt{\frac{n}{d}} - 1\right)}{2} + \frac{1}{\sqrt{\frac{n}{d}} + 1} \ge \frac{1}{2} - \frac{1}{2}\sqrt{\frac{d}{n}} +  \frac{1}{2}\sqrt{\frac{d}{n}} = \frac{1}{2}.
    \end{align*}
    Thus, by~\eqref{eq:pd-u-large} we have
    \begin{align*}
        P_d \ge \frac{CC_{\beta, u}}{2}.
    \end{align*}

    \textbf{If} $\beta \in (1, 2)$, then by the third statement of Lemma~\ref{lem:sum-int} we have
    \begin{align*}
        \sum_{i = 1}^u f(i) &\ge f(1) + f(u) - f\left(\sqrt{\frac{n}{d}}\right) + \int_1^u f(x) dx \\
        &= f(1) + \int_{\sqrt{\frac{n}{d}}}^u f'(x) dx + \int_1^{\sqrt{\frac{n}{d}}} f(x) dx + \int_{\sqrt{\frac{n}{d}}}^u f(x) dx \\
        &= \frac{d}{n} + \frac{d}{n} \cdot \frac{\sqrt{\frac{n}{d}}^{3 - \beta} - 1}{3 - \beta} + \int_{\sqrt{\frac{n}{d}}}^u (x^{-\beta} - \beta x^{-(\beta + 1)}) dx \\
        &= \frac{d}{n} \cdot \frac{(2 - \beta)}{(3 - \beta)} + \frac{\sqrt{\frac{n}{d}}^{1 - \beta}}{3 - \beta} + \int_{\sqrt{\frac{n}{d}}}^u (x^{-\beta} - \beta x^{-(\beta + 1)}) dx \\
        &\ge \frac{\sqrt{\frac{n}{d}}^{1 - \beta}}{3 - \beta} + \int_{\sqrt{\frac{n}{d}}}^u (x^{-\beta} - \beta x^{-(\beta + 1)}) dx.
    \end{align*}
    Note that if $x > \beta$, then the function in the integral is non-negative. Thus, if $\sqrt{\frac{n}{d}} \ge \beta$, then the integral is non-negative itself and then we have
    \begin{align*}
        \sum_{i = 1}^u f(i) &\ge\frac{\sqrt{\frac{n}{d}}^{1 - \beta}}{3 - \beta} \ge \frac{1}{2}\sqrt{\frac{n}{d}}^{1- \beta}.
    \end{align*}
    Otherwise, if $\sqrt{\frac{n}{d}} < \beta$, then we have
    \begin{align*}
        \sum_{i = 1}^u f(i) \ge f(1) + f(2),
    \end{align*}
    since we consider the case $u > \sqrt{\frac{n}{d}} \ge 1$ and $u$ is integer, hence $u \ge 2$. We have $f(2) = 2^{-\beta} \min\{1, \frac{4d}{n}\} = 2^{-\beta}$, since $\frac{d}{n} > \frac{1}{\beta^2} \ge \frac{1}{4}$. Thus, we have
    \begin{align*}
        \sum_{i = 1}^u f(i) &\ge f(1) + f(2) = \frac{d}{n} + 2^{-\beta}= \sqrt{\frac{n}{d}}^{1- \beta} \left(\sqrt{\frac{n}{d}}^{\beta - 3} + 2^{-\beta}{\frac{n}{d}}^{1- \beta} \right) \\
        &\ge \left(\beta^{\beta - 3} + 2^{-\beta}\right) \sqrt{\frac{n}{d}}^{1- \beta} \ge \left(\frac{1}{4} + \frac{1}{4}\right) \sqrt{\frac{n}{d}}^{1- \beta} = \frac{1}{2}\sqrt{\frac{n}{d}}^{1- \beta}.
    \end{align*}
    Hence, in both cases we have
    \begin{align*}
        P_d \ge CC_{\beta, u} \cdot \frac{1}{2}\sqrt{\frac{n}{d}}^{1 - \beta}.
    \end{align*}
    
    \textbf{If} $\beta \in [2, 3)$, then by the second statement of Lemma~\ref{lem:sum-int} we have
    \begin{align*}
        \sum_{i = 1}^u f(i) &\ge \int_1^{\sqrt{\frac{n}{d}}} \frac{d}{n} x^{2 - \beta} dx + \int_{\sqrt{\frac{n}{d}}}^u x^{-\beta} dx + u^{-\beta} \\
        &\ge \int_1^{\sqrt{\frac{n}{d}}} \frac{d}{n} x^{2 - \beta} dx = \frac{d}{n} \cdot \frac{\left(\sqrt{\frac{n}{d}}^{3 - \beta} - 1\right)}{(3 - \beta)} = \frac{\sqrt{\frac{n}{d}}^{1 - \beta} - \sqrt{\frac{n}{d}}^{-2}}{3 - \beta} \\
        &= \frac{\sqrt{\frac{n}{d}}^{1 - \beta}}{3 - \beta} \left(1 - \sqrt{\frac{n}{d}}^{\beta - 3}\right).
    \end{align*}
    If $\sqrt{\frac{n}{d}}^{\beta - 3} \le \frac{1}{2}$, then this is at least
    \begin{align*}
        \sum_{i = 1}^u f(i) &\ge \frac{\sqrt{\frac{n}{d}}^{1 - \beta}}{3 - \beta} \left(1 - \frac{1}{2}\right) = \frac{\sqrt{\frac{n}{d}}^{1 - \beta}}{2(3 - \beta)} \ge \frac{1}{2} \sqrt{\frac{n}{d}}^{1 - \beta}.
    \end{align*}
    If $\sqrt{\frac{n}{d}}^{\beta - 3} > \frac{1}{2}$, then we have
    \begin{align*}
        \sum_{i = 1}^u f(i) &\ge f(1) = \frac{d}{n} = \sqrt{\frac{n}{d}}^{1 - \beta} \sqrt{\frac{n}{d}}^{\beta - 3} \ge \frac{1}{2} \sqrt{\frac{n}{d}}^{1 - \beta}.
    \end{align*}

    Putting these estimates into~\eqref{eq:pd-u-large}, we obtain
    \begin{align*}
        P_d \ge CC_{\beta, u} \cdot \frac{1}{2} \sqrt{\frac{n}{d}}^{1 - \beta}.
    \end{align*}

    \textbf{If} $\beta = 3$, then by the second statement of Lemma~\ref{lem:sum-int} we have
    
    \begin{align*}
        \sum_{i = 1}^u f(i) &\ge \frac{d}{n}\int_1^{\sqrt{\frac{n}{d}}} \frac{dx}{x} + \int_{\sqrt{\frac{n}{d}}}^u \frac{dx}{x^3} + u^{-3}.
    \end{align*}
    If we consider the last two terms as a function of $u$, we can see that it is non-decreasing if $u \ge 3$, since its derivative is
    \begin{align*}
        \left(\int_{\sqrt{\frac{n}{d}}}^u \frac{dx}{x^3} + u^{-3}\right)'_u = \frac{1}{u^3} - \frac{3}{u^4} = \frac{u - 3}{u^4}.
    \end{align*}
    Hence, if $\sqrt{\frac{n}{d}} \ge 3$, then the lower bound is minimized by taking $u = \sqrt{\frac{n}{d}}$, which implies
    \begin{align*}
        \sum_{i = 1}^u f(i) &\ge \frac{d}{n}\int_1^{\sqrt{\frac{n}{d}}} \frac{dx}{x} + \sqrt{\frac{n}{d}}^{-3} = \frac{d}{n}\ln\sqrt{\frac{n}{d}} + \frac{d}{n} \frac{1}{\sqrt{\frac{n}{d}}} \\
        &\ge \frac{d}{n} \ln\left(\sqrt{\frac{n}{d}} \left(1 + \frac{1}{\sqrt{\frac{n}{d}}}\right)\right) = \frac{d}{n} \ln\left(\sqrt{\frac{n}{d}} + 1\right),
    \end{align*}
    where we used inequality $x \ge \ln(1 + x)$, which holds for all $x > -1$. Otherwise, when $\sqrt{\frac{n}{d}} \in [2, 3)$,
    then we estimate the lower bound on the sum as the sum of the first two terms. Note, that there are at least two terms in the sum, since we are considering the case $u > \sqrt{\frac{n}{d}} \ge 1$ and $u$ is an integer, hence $u \ge 2$.
    \begin{align*}
        \sum_{i = 1}^u f(i) &\ge f(1) + f(2) = \frac{d}{n} + \frac{1}{8}\min\left\{1, \frac{4d}{n}\right\} = \frac{d}{n} + \frac{d}{2n} \\
        &= \frac{3d}{2n} \ge \frac{d}{n} \ln(4) \ge \frac{d}{n} \ln\left(\sqrt{\frac{n}{d}} + 1\right).
    \end{align*} 
    If $\sqrt{\frac{n}{d}} \in [1, 2)$, then we get the same lower bound.
    \begin{align*}
        \sum_{i = 1}^u f(i) &= f(1) + f(2) = \frac{d}{n} + \frac{1}{8}\min\left\{1, \frac{4d}{n}\right\} = \frac{d}{n} + \frac{1}{8} \\
        &\ge \frac{9d}{8n} \ge \frac{d}{n} \ln(3) \ge \frac{d}{n} \ln\left(\sqrt{\frac{n}{d}} + 1\right).
    \end{align*} 

    Putting this lower bound into~\eqref{eq:pd-u-large}, we obtain
    \begin{align*}
        P_d \ge CC_{\beta, u} \cdot \frac{d}{n} \ln\left(\sqrt{\frac{n}{d}} + 1\right).
    \end{align*}

    \textbf{If} $\beta > 3$, we have
    \begin{align*}
        \sum_{i = 1}^u f(i) \ge f(1) = \frac{d}{n}.
    \end{align*}
    Therefore,
    \begin{align*}
        P_d \ge CC_{\beta, u} \cdot \frac{d}{n}.
    \end{align*}
\end{proof}

We proceed with the following upper bound on the expected cost of one iteration.

\begin{lemma}\label{lem:expected-lambda}
    The expected cost of one iteration of the heavy-tailed \ollga is as shown in Table~\ref{tbl:iter-cost} 
\end{lemma}
\begin{table}
    \caption{Upper bounds on the cost of one iteration of the heavy-tailed \ollga (which is, $E[2\lambda]$) depending on the parameters of the power-law distribution. $C_{\beta, u}$ stands for the normalization constant of the power-law distribution.}
	\label{tbl:iter-cost}
	\begin{center}
		\begin{tabular}{|c|l|}
			\hline
            & \\[-10pt]
            $\beta$ & 
            $E[2\lambda]$ \\[5pt] \hline
            & \\[-10pt]
            $< 1$ & 
            $2(3 - \beta)C_{\beta, u}\frac{u^{2 - \beta}}{2 - \beta}$ \\[5pt] \hline
            & \\[-10pt]
            $[1, 2)$ & 
            $2C_{\beta, u}\frac{u^{2 - \beta}}{2 - \beta}$ \\[5pt] \hline
            & \\[-10pt]
            $2$ & 
            $2C_{\beta, u}(\ln(u) + 1)$ \\[5pt] \hline
            & \\[-10pt]
            $>2$ & 
            $2C_{\beta, u}\frac{\beta - 1}{\beta - 2}$ \\[5pt] \hline
        \end{tabular}
	\end{center}
\end{table}

\begin{proof}
    Since after choosing $\lambda$ the cost of one iteration is $2\lambda$, the expected cost of one iteration is $E[2\lambda]$, which we compute as follows.
    \begin{align}\label{eq:expected-lambda}
        E[2\lambda] = \sum_{i = 1}^{u} 2i \Pr[\lambda = i] = 2C_{\beta, u} \sum_{i = 1}^u i^{1 - \beta}.
    \end{align}
    We now estimate the sum for different values of $\beta$.
    
    \textbf{If} $\beta < 1$, then by Lemma~\ref{lem:part-sum} we have
    \begin{align*}
        \sum_{i = 1}^u i^{1 - \beta} &\le \frac{u^{2 - \beta} - 1}{2 - \beta} + u^{1 - \beta} = \frac{u^{2 - \beta}}{2 - \beta} \left(1 - u^{\beta - 2} + (2 - \beta)u^{-1}\right) \\
        &\le \frac{u^{2 - \beta}}{2 - \beta} (1 + (2 - \beta)) = \frac{(3 - \beta)u^{2 - \beta}}{2 - \beta}.
    \end{align*}
    Hence, by~\eqref{eq:expected-lambda} we have
    \begin{align*}
        E[2\lambda] \le 2(3 - \beta)C_{u, \beta}\frac{u^{2 - \beta}}{2 - \beta}.
    \end{align*}

    \textbf{If} $\beta \in [1, 2)$, then by Lemma~\ref{lem:part-sum} we have
    \begin{align*}
        \sum_{i = 1}^u i^{1 - \beta} &\le \frac{u^{2 - \beta} - 1}{2 - \beta} + 1 = \frac{u^{2 - \beta}}{2 - \beta} - \frac{\beta - 1}{2 - \beta} \le \frac{u^{2 - \beta}}{2 - \beta}.
    \end{align*}
    Therefore, by~\eqref{eq:expected-lambda} we compute
    \begin{align*}
        E[2\lambda] \le 2C_{u, \beta}\frac{u^{2 - \beta}}{2 - \beta}.
    \end{align*}

    \textbf{If} $\beta = 2$, then by Lemma~\ref{lem:part-sum} we have
    \begin{align*}
        \sum_{i = 1}^u i^{1 - \beta} &\le \ln(u) + 1.
    \end{align*}
    Consequently, by~\eqref{eq:expected-lambda} we have
    \begin{align*}
        E[2\lambda] \le 2C_{u, \beta}(\ln(u) + 1).
    \end{align*}
    
    \textbf{If} $\beta > 2$, then by Lemma~\ref{lem:part-sum} we have
    \begin{align*}
        \sum_{i = 1}^u i^{1 - \beta} &\le \frac{1 - u^{2 - \beta}}{\beta - 2} + 1 \le \frac{1}{\beta - 2} + 1 = \frac{\beta - 1}{\beta - 2}.
    \end{align*}
    Hence, by~\eqref{eq:expected-lambda} we have
    \begin{align*}
        E[2\lambda] \le 2C_{u, \beta}\frac{\beta - 1}{2 - \beta}.
    \end{align*}
\end{proof}

We are now in position to prove Theorem~\ref{thm:fast}. Before we start, we sketch the general idea of the proof. Below we denote by $T_I$ the number of iterations which the algorithm makes before it finds the optimum and by $T_F$ we denote the corresponding expected number of fitness evaluations. 
    
We split the proof into three cases depending on the parameter $u$. We start the proof of each case with finding an upper bound on $E[T_I]$. If we denote by $T_d$ the number of iterations which the algorithm spends in distance $d$ from the optimum, the total number of iterations $T$ is then the sum of $T_d$ for all distances $d \in [1..D]$. Note that for all $d > D$ we have $T_d = 0$ due to the elitist selection of the \ollga. Also the elitism implies that once we leave distance level $d$ we never return to it, hence each $T_d$ is dominated by the geometric distribution $\Geom(P_d)$, where $P_d$ is the probability of a successful iteration when we are in distance $d$ from the optimum. Hence, we have
\begin{align*}
    E[T_I] = \sum_{d = 1}^D E[T_d] \le \sum_{d = 1}^D P_d^{-1}.
\end{align*}

Once we find $E[T_I]$, we can find $E[T_F]$ using the Wald's equation (Lemma~\ref{lem:wald}). For this we note that
\begin{align*}
    E[T_F] = E\left[\sum_{t = 1}^{T_I} 2\lambda_t \right],
\end{align*}
where $\lambda_t$ is the value of $\lambda$ chosen in iteration $t$. We now check that these values satisfy the conditions of Wald's equation.
\begin{enumerate}
    \item By Lemma~\ref{lem:expected-lambda} all $\lambda_t$ (and therefore, $2\lambda_t$) have the same finite expectation.
    \item The property $E[\lambda_t \mathds{1}_{T_I \ge t}] = E[\lambda_t]\Pr[T_I \ge t]$ follows from the fact that we choose $\lambda$ independently of the iteration. 
    \item Since $\lambda_t$ is non-negative, for the third property we have 
    \begin{align*}
        \sum_{t = 1}^{+\infty} E[|\lambda_t| \mathds{1}_{T_I \ge t}] &= \sum_{t = 1}^{+\infty} E[\lambda_t \mathds{1}_{T_I \ge t}] = \sum_{t = 1}^{+\infty} E[\lambda_t] \Pr[T_I \ge t] \\
        &= E[\lambda] E[T_I] < \infty.
    \end{align*} 
    \item In the proves below we show that $E[T_I]$ is always finite.
\end{enumerate}

Therefore, we can compute $E[T_F] = E[2\lambda]E[T_I]$ using our estimates for $E[T_I]$ and Lemma~\ref{lem:expected-lambda} for $E[2\lambda]$. 

\begin{proof}[Proof of Theorem~\ref{thm:fast}]

    \textbf{Case 1: $u \le \sqrt{\frac{n}{D}}$.}

    In this case we also have $u \le \sqrt{\frac{n}{d}}$ for all $d \in [1..D]$. Therefore we take the values for $P_d$ from the left column of Table~\ref{tbl:pd}.
    
    \textbf{If} $\beta < 3$, then we have
    \begin{align*}
        E[T_I] &\le \sum_{d = 1}^D \frac{\max\{1, 3 - \beta\}}{CC_{\beta, u}} \cdot \frac{n}{du^{3 - \beta}} = \frac{\max\{1, 3 - \beta\}}{CC_{\beta, u}} \cdot \frac{n}{u^{3 - \beta}} \sum_{d = 1}^D \frac{1}{d} \\
        &\le \frac{\max\{1, 3 - \beta\}}{CC_{\beta, u}} \cdot \frac{n}{u^{3 - \beta}} \left(\ln(D) + 1\right),
    \end{align*}
    where in the last inequality we used Lemma~\ref{lem:part-sum}.
    \textbf{If} $\beta = 3$, then by Lemma~\ref{lem:part-sum} we have
    \begin{align*}
        E[T_I] &\le \sum_{d = 1}^D \frac{1}{CC_{\beta, u}} \cdot \frac{n}{d\ln(u + 1)} =  \frac{1}{CC_{\beta, u}} \cdot \frac{n}{\ln(u + 1)} \sum_{d = 1}^D \frac{1}{d} \\
        &\le  \frac{1}{CC_{\beta, u}} \cdot \frac{n}{\ln(u + 1)} \left(\ln(D) + 1\right).
    \end{align*}
    \textbf{If} $\beta > 3$, then by Lemma~\ref{lem:part-sum} we have
    \begin{align*}
        E[T_I] &\le \sum_{d = 1}^D \frac{1}{CC_{\beta, u}} \cdot \frac{n}{d} =  \frac{n}{CC_{\beta, u}}  \sum_{d = 1}^D \frac{1}{d} \\
        &\le  \frac{n}{CC_{\beta, u}} \left(\ln(D) + 1\right).
    \end{align*}

    We now use the Wald's equation (Lemma~\ref{lem:wald}) and the estimates of $E[2\lambda]$ from Lemma~\ref{lem:expected-lambda} to compute $E[T_F] = E[T_I]E[2\lambda]$.

    If $\beta < 2$, then we have
    \begin{align*}
        E[T_F] &\le \left(2(3 - \beta)C_{\beta, u} \frac{u^{2 - \beta}}{2 - \beta}\right) \left(\frac{3 - \beta}{CC_{\beta, u}} \cdot \frac{n}{u^{3 - \beta}} \left(\ln(D) + 1\right)\right) \\
        &= \frac{2(3 - \beta)^2}{C(2 - \beta)} \cdot \frac{n}{u}(\ln(D) + 1).
    \end{align*}

    If $\beta = 2$, then we have
    \begin{align*}
        E[T_F] &\le \left(2C_{2, u} (\ln(u) + 1)\right) \left(\frac{1}{CC_{2, u}} \cdot \frac{n}{u} \left(\ln(D) + 1\right)\right) \\
        &= \frac{2}{C} \cdot \frac{n(\ln(u) + 1)}{u} (\ln(D) + 1).
    \end{align*}

    If $\beta \in (2, 3)$, then we have
    \begin{align*}
        E[T_F] &\le \left(2C_{\beta, u} \frac{\beta - 1}{\beta - 2}\right) \left(\frac{1}{CC_{\beta, u}} \cdot \frac{n}{u^{3 - \beta}} \left(\ln(D) + 1\right)\right) \\
        &= \frac{2(\beta - 1)}{C(\beta -2)} \cdot \frac{n}{u^{3 - \beta}} (\ln(D) + 1).
    \end{align*}

    If $\beta =3$, then we have
    \begin{align*}
        E[T_F] &\le \left(2C_{\beta, u} \frac{\beta - 1}{\beta - 2}\right) \left(\frac{1}{CC_{\beta, u}} \cdot \frac{n}{\ln(u + 1)} \left(\ln(D) + 1\right)\right) \\
        &= \frac{2(\beta - 1)}{C(\beta -2)} \cdot \frac{n}{\ln(u + 1)} (\ln(D) + 1).
    \end{align*}

    If $\beta >3$, then we have
    \begin{align*}
        E[T_F] &\le \left(2C_{\beta, u} \frac{\beta - 1}{\beta - 2}\right) \left(\frac{n}{CC_{\beta, u}} \left(\ln(D) + 1\right)\right) \\
        &= \frac{2(\beta - 1)}{C(\beta -2)} \cdot n (\ln(D) + 1).
    \end{align*}

\textbf{Case 2: $u \in [\sqrt{\frac{n}{D}}, \sqrt{n}]$.}

    By $P_d'$ we denote the lower bounds on $P_d$ obtained in Lemma~\ref{lem:pd} and shown in Table~\ref{tbl:pd}. For all values of $\beta$ we aim at showing that $(P_d')^{-1}$ is non-increasing in $d$ and then use Lemma~\ref{lem:sum-int} to estimate the sum of $(P_d')^{-1}$.
    
    \textbf{If} $\beta \le 1$, then we denote
    \begin{align*}
        P_d' \coloneqq \begin{cases}
            \frac{CC_{\beta, u}}{3 - \beta} \cdot \frac{du^{3 - \beta}}{n}, &\text{ if } d \le \frac{n}{u^2}, \\
            \frac{CC_{\beta, u}}{3 - \beta} \cdot u^{1 - \beta}, &\text{ if } d > \frac{n}{u^2}. \\
        \end{cases}
    \end{align*}
    When $d \le \frac{n}{u^2}$ this is an increasing function of $d$, and it is at most 
    \begin{align*}
        \frac{CC_{\beta, u}}{3 - \beta} \cdot \frac{u^{3 - \beta}}{n} \cdot \frac{n}{u^2} =  \frac{CC_{\beta, u}}{3 - \beta} \cdot u^{1 - \beta}.
    \end{align*}
    For larger $d$ it is a constant function of $d$. Hence, $P_d'$ is a positive non-decreasing function of $d$ for $d \in [1, n]$ (which is a real-valued interval), which implies that $(P_d')^{-1}$ is a positive non-increasing function. Thus, by Lemma~\ref{lem:sum-int} we compute
    \begin{align*}
        E[T_I] &\le \sum_{d = 1}^D (P_d')^{-1} \le P_1'^{-1}  + \int_{1}^{\frac{n}{u^2}} \frac{3 - \beta}{CC_{\beta, u}} \cdot \frac{n}{xu^{3 - \beta}} dx + \int_{\frac{n}{u^2}}^D \frac{3 - \beta}{CC_{\beta, u}} \cdot u^{\beta - 1} dx \\
        &= \frac{3 - \beta}{CC_{\beta,u}} \cdot \left(\frac{n}{u^{3 - \beta}} + \frac{n}{u^{3 - \beta}} \ln\left(\frac{n}{u^2}\right)  + u^{\beta - 1}\left(D - \frac{n}{u^2}\right)\right) \\
        &= \frac{3 - \beta}{CC_{\beta,u}} \cdot \frac{1}{u^{1 - \beta}} \cdot \left(\frac{n}{u^2}\ln\left(\frac{n}{u^2}\right) + D\right).
    \end{align*}

    \textbf{If} $\beta \in (1, 3)$, then we denote
    \begin{align*}
        P_d' \coloneqq \begin{cases}
            \frac{CC_{\beta, u}}{4} \cdot \frac{du^{3 - \beta}}{n}, &\text{ if } d \le \frac{n}{u^2}, \\
            \frac{CC_{\beta, u}}{4} \cdot \sqrt{\frac{n}{d}}^{1 - \beta}, &\text{ if } d > \frac{n}{u^2}. \\
        \end{cases}
    \end{align*}
    We note that $P_d'$ is increasing in $d$ both in $[1, \frac{n}{u^2}]$ and in $[\frac{n}{u^2}, D]$. Also for all $d \le \frac{n}{u^2}$ we have
    \begin{align*}
        P_d' \le \frac{CC_{\beta, u}}{4} \cdot \frac{u^{3 - \beta}}{n} \cdot \frac{n}{u^2} = \frac{CC_{\beta, u}}{4} \cdot u^{1 - \beta},
    \end{align*}
    and for all $d > \frac{n}{u^2}$ we have
    \begin{align*}
        P_d' > \frac{CC_{\beta, u}}{4} \cdot \sqrt{n}^{1 - \beta} \sqrt{\frac{n}{u^2}}^{\beta - 1} = \frac{CC_{\beta, u}}{4} \cdot u^{1 - \beta}.
    \end{align*}
    Therefore, $P_d'$ is increasing in all interval $[1, D]$. Hence, by Lemma~\ref{lem:sum-int} we have
    \begin{align*}
        E[T_I] &\le \sum_{d = 1}^D (P_d')^{-1} \le P_1' + \int_{1}^{\frac{n}{u^2}} \frac{4}{CC_{\beta, u}} \cdot \frac{n}{xu^{3 - \beta}}dx + \int_{\frac{n}{u^2}}^{D} \frac{4}{CC_{\beta, u}} \sqrt{\frac{x}{n}}^{1 - \beta}dx \\
        &=\frac{4}{CC_{\beta, u}} \left(\frac{n}{u^{3 - \beta}} + \frac{n}{u^{3 - \beta}} \ln\left(\frac{n}{u^{2}}\right) + \sqrt{n}^{\beta - 1} \cdot \frac{2\sqrt{x}^{3 - \beta}}{3 - \beta}\bigg\rvert_{\frac{n}{u^2}}^D \right) \\
        &= \frac{4}{CC_{\beta, u}} \left(\frac{n}{u^{3 - \beta}} \left(\ln\left(\frac{n}{u^{2}}\right) + 1 - \frac{2}{3 - \beta}\right) + \frac{2\sqrt{n}^{\beta - 1}\sqrt{D}^{3 - \beta}}{3 - \beta}  \right) \\
        &\le \frac{4}{CC_{\beta, u}} \left(\frac{n}{u^{3 - \beta}} \ln\left(\frac{n}{u^{2}}\right) + \frac{2\sqrt{n}^{\beta - 1}\sqrt{D}^{3 - \beta}}{3 - \beta}  \right).
    \end{align*}

    \textbf{If} $\beta = 3$, then we have
    \begin{align*}
        P_d' \coloneqq \begin{cases}
            CC_{\beta, u} \cdot \frac{d\ln(u + 1)}{n}, &\text{ if } d \le \frac{n}{u^2}, \\
            CC_{\beta, u} \cdot \frac{d\ln\left(\sqrt{\frac{n}{d}} + 1\right)}{n}, &\text{ if } d > \frac{n}{u^2}. \\
        \end{cases}
    \end{align*}
    In the first case we have a linear function of $d$. To show that in the second case we have an increasing function in $d$, we consider the its derivative over $d$.
    \begin{align*}
        \left(d\ln\left(\sqrt{\frac{n}{d}} + 1\right)\right)_d' &= \ln\left(\sqrt{\frac{n}{d}} + 1\right) + \frac{d}{\sqrt{\frac{n}{d}} + 1} \cdot \left(-\frac{\sqrt{n}}{2\sqrt{d}^3}\right) \\
        &= \ln\left(\sqrt{\frac{n}{d}} + 1\right) - \frac{1}{2\left(\sqrt{\frac{d}{n}} + 1\right)} \\
        &\ge \ln(1 + 1) - \frac{1}{2} > 0. 
    \end{align*} 
    Thus, $\frac{d\ln(\sqrt{\frac{n}{d}} + 1)}{n}$ is an increasing function of $d$. To show that $P_d'$ is an increasing function in all interval $[1, D]$, we note that when $d \le \frac{n}{u^2}$, we have
    \begin{align*}
        P_d' \le CC_{\beta, u} \frac{\ln(u + 1)}{u^2},
    \end{align*}
    and when $d > \frac{n}{u^2}$, we have
    \begin{align*}
        P_d' > CC_{\beta, u} \frac{\ln(u + 1)}{u^2}.
    \end{align*}
    Consequently, by Lemma~\ref{lem:sum-int} we obtain
    \begin{align*}
        E[T_I] &\le \sum_{d = 1}^D (P_d')^{-1} \\
        &\le P_1' + \int_{1}^{\frac{n}{u^2}} \frac{1}{CC_{\beta, u}} \cdot \frac{n}{x\ln(u + 1)}dx + \int_{\frac{n}{u^2}}^{D} \frac{1}{CC_{\beta, u}} \cdot \frac{n}{x\ln\left(\sqrt{\frac{n}{x}} + 1\right)}dx \\
        &\le \frac{1}{CC_{\beta, u}} \left(\frac{n}{\ln(u + 1)} + \frac{n}{\ln(u + 1)} \ln\left(\frac{n}{u^2}\right) + \int_{\frac{n}{u^2}}^D \frac{n}{x\ln\left(\sqrt{\frac{n}{x}} + 1\right)}dx\right).
    \end{align*}
    We estimate the integral as follows.
    \begin{align*}
        \int_{\frac{n}{u^2}}^D \frac{n}{x\ln\left(\sqrt{\frac{n}{x}} + 1\right)}dx &\le  \int_{\frac{n}{u^2}}^D \frac{n}{x\ln\left(\sqrt{\frac{n}{x}}\right)}dx = \left[\begin{array}{l}
            t = \sqrt{\frac{n}{x}} \\
            x = \frac{n}{t^2} \\
            dx = -\frac{2n}{t^3} dt
        \end{array}\right] \\
        &= \int_{u}^{\sqrt{\frac{n}{D}}} \frac{n\cdot\left(-\frac{2n}{t^3}\right) dt}{\frac{n}{t^2}\ln(t)} = -2n \int_{u}^{\sqrt{\frac{n}{D}}} \frac{dt}{t\ln(t)} \\
        &= -2n \ln\ln(t) \bigg\rvert_u^{\sqrt{\frac{n}{D}}} = 2n\ln\left(\frac{\ln(u)}{\ln(\sqrt{n/D})}\right).
    \end{align*}
    This results in a very tight bound for small values of $D$. E.g., if $D$ is constant and $u$ just a little bit greater than $\sqrt{\frac{n}{D}}$, we have
    \begin{align*}
        E[T_I] \le \Theta\left(\frac{n}{\log(n)}\right).
    \end{align*}
    
    To avoid an infinite upper bound for large $D$ we also show a different estimate of this integral for $D > \frac{n}{e^2}$ depending on the value of $u$. If $u \ge e$ (and thus $\frac{n}{u^2} \le \frac{n}{e^2}$), we have
    \begin{align*}
        \int_{\frac{n}{u^2}}^D \frac{n}{x\ln\left(\sqrt{\frac{n}{x}} + 1\right)}dx &\le \int_{\frac{n}{u^2}}^{\frac{n}{e^2}} \frac{n}{x\ln\left(\sqrt{\frac{n}{x}}\right)}dx + \int_{\frac{n}{e^2}}^D \frac{n}{x\ln(2)}dx \\
        &= 2n\ln \left(\frac{\ln(u)}{\ln(e)}\right) + \frac{n}{\ln(2)} \ln\left(\frac{De^2}{n}\right) \\
        &= 2n\ln\ln(u) + \frac{2n}{\ln(2)}.
    \end{align*}
    If $u < e$, we have
    \begin{align*}
        \int_{\frac{n}{u^2}}^D \frac{n}{x\ln\left(\sqrt{\frac{n}{x}} + 1\right)}dx &\le \int_{\frac{n}{u^2}}^D \frac{n}{x\ln(2)}dx = \frac{n}{\ln(2)} \ln\left(\frac{Du^2}{n}\right) \\
        &\le \frac{n}{\ln(2)}\ln(u^2) < \frac{2n}{\ln(2)}.
    \end{align*}
    Hence, for $D > \frac{n}{e^2}$ we have
    \begin{align*}
        \int_{\frac{n}{u^2}}^D \frac{n}{x\ln\left(\sqrt{\frac{n}{x}} + 1\right)}dx &\le 2n\ln\ln^+(u) + \frac{2n}{\ln(2)},
    \end{align*}
    where by $\ln^+(x)$ we denote $\max\{1, \ln(x)\}$.

    Wrapping up the two cases, for $D \le \frac{n}{e^2}$ we have
    \begin{align*}
        E[T_I] &\le \frac{1}{CC_{\beta, u}} \left(\frac{n}{\ln(u + 1)} \left(\ln\left(\frac{n}{u^2}\right) + 1\right) + 2n\ln\left(\frac{\ln(u)}{\ln(\sqrt{n/D})}\right)\right),
    \end{align*}
    and for $D > \frac{n}{e^2}$ we have 
    \begin{align*}
        E[T_I] &\le \frac{1}{CC_{\beta, u}} \left(\frac{n}{\ln(u + 1)} \left(\ln\left(\frac{n}{u^2}\right) + 1\right) + 2n\ln\ln^+(u) + \frac{2n}{\ln(2)}\right).
    \end{align*}

    \textbf{If} $\beta > 3$, then we have $P_d' \coloneqq CC_{\beta, u} \frac{d}{n}$. Hence, by Lemma~\ref{lem:part-sum} we have
    \begin{align*}
        E[T_I]&\le \sum_{d = 1}^D (P_d')^{-1} = \frac{n}{CC_{\beta, u}} \sum_{d = 1}^D \frac{1}{d} \le \frac{n}{CC_{\beta, u}} \left(\ln(D) + 1\right).
    \end{align*}

    Now we transform the estimates on the expectation of $T_I$ into the expectation of $T_F$ using the Wald's equation (Lemma~\ref{lem:wald}) and the estimates for $E[2\lambda]$ from Lemma~\ref{lem:expected-lambda}.

    \textbf{If} $\beta \le 1$, then we have
    \begin{align*}
        E[T_F] &= E[T_I]E[2\lambda] \le \left(2(3 - \beta)C_{\beta, u} \frac{u^{2 - \beta}}{2 - \beta}\right) \\
        &\cdot \left(\frac{3 - \beta}{CC_{\beta,u}} \cdot \frac{1}{u^{1 - \beta}} \cdot \left(\frac{n}{u^2}\ln\left(\frac{n}{u^2}\right) + D\right)\right) \\
        &= \frac{2(3 - \beta)^2}{C(2 - \beta)} \cdot \left(\frac{n}{u} \ln\left(\frac{n}{u^2}\right) + Du\right).
    \end{align*}

    \textbf{If} $\beta \in (1, 2)$, then we have
    \begin{align*}
        E[T_F] &= E[T_I]E[2\lambda] \le \left(2(3 - \beta)C_{\beta, u} \frac{u^{2 - \beta}}{2 - \beta}\right) \\
        &\cdot \left(\frac{4}{CC_{\beta, u}} \left(\frac{n}{u^{3 - \beta}} \ln\left(\frac{n}{u^{2}}\right) + \frac{2\sqrt{n}^{\beta - 1}\sqrt{D}^{3 - \beta}}{3 - \beta}  \right)\right) \\
        &= \frac{8(3 - \beta)}{C(2 - \beta)} \left(\frac{n}{u} \ln\left(\frac{n}{u^2}\right) + \frac{2}{3 - \beta}\sqrt{n}^{\beta - 1}\sqrt{D}^{3 - \beta} u^{2 - \beta}\right) \\
        &= \frac{8(3 - \beta)}{C(2 - \beta)} \cdot \frac{n}{u} \left(\ln\left(\frac{n}{u^2}\right) + \frac{2}{3 - \beta}\left(u\sqrt{\frac{D}{n}}\right)^{3 - \beta}\right).
    \end{align*}

    \textbf{If} $\beta = 2$, then we have
    \begin{align*}
        E[T_F] &= E[T_I]E[2\lambda] \le \left(2C_{2, u} (\ln(u) + 1)\right) \\
        &\cdot \left(\frac{4}{CC_{2, u}} \cdot  \left(\frac{n}{u} \ln\left(\frac{n}{u^2}\right) + 2 \sqrt{nD} \right) \right) \\
        &= \frac{8}{C} (\ln(u) + 1)\left(\frac{n}{u} \ln\left(\frac{n}{u^2}\right) +  2\sqrt{nD}\right)
    \end{align*}

    \textbf{If} $\beta \in (2, 3)$, then we have
    \begin{align*}
        E[T_F] &= E[T_I]E[2\lambda] \le \left(2C_{\beta, u} \frac{\beta - 1}{\beta - 2}\right) \\
        &\cdot \left(\frac{4}{CC_{\beta, u}} \left(\frac{n}{u^{3 - \beta}} \ln\left(\frac{n}{u^{2}}\right) + \frac{2\sqrt{n}^{\beta - 1}\sqrt{D}^{3 - \beta}}{3 - \beta}  \right)\right) \\
        &= \frac{8(\beta - 1)}{C(\beta - 2)} \cdot\frac{n}{u^{3 - \beta}} \left(\ln\left(\frac{n}{u^2}\right) + \frac{2}{3 - \beta}\left(u\sqrt{\frac{D}{n}}\right)^{3 - \beta}\right)
    \end{align*}

    \textbf{If} $\beta =3$ \textbf{and} $D \le \frac{n}{e^2}$, then we have
    \begin{align*}
        E[T_F] &= E[T_I]E[2\lambda] \le \left(2C_{\beta, u} \frac{\beta - 1}{\beta - 2}\right) \cdot \frac{1}{CC_{\beta, u}} \\
        &\cdot \left(\frac{n}{\ln(u + 1)} \left(\ln\left(\frac{n}{u^2}\right) + 1\right) + 2n\ln\left(\frac{\ln(u)}{\ln(\sqrt{n/D})}\right)\right) \\
        &= \frac{2(\beta - 1)}{C(\beta - 2)} \cdot \left(\frac{n}{\ln(u + 1)} \left(\ln\left(\frac{n}{u^2}\right) + 1\right) + 2n\ln\left(\frac{\ln(u)}{\ln(\sqrt{n/D})}\right)\right).
    \end{align*}

    \textbf{If} $\beta =3$ \textbf{and} $D > \frac{n}{e^2}$, then we have
    \begin{align*}
        E[T_F] &= E[T_I]E[2\lambda] \le \left(2C_{\beta, u} \frac{\beta - 1}{\beta - 2}\right) \cdot \frac{1}{CC_{\beta, u}} \\
        &\cdot \left(\frac{n}{\ln(u + 1)} \left(\ln\left(\frac{n}{u^2}\right) + 1\right) + 2n\ln\ln^+(u) + \frac{2n}{\ln(2)}\right) \\
        &= \frac{2(\beta - 1)}{C(\beta - 2)} \cdot \left(\frac{n}{\ln(u + 1)} \left(\ln\left(\frac{n}{u^2}\right) + 1\right) + 2n\ln\ln^+(u) + \frac{2n}{\ln(2)}\right).
    \end{align*}

    \textbf{If} $\beta >3$, then we have
    \begin{align*}
        E[T_F] &\le \left(2C_{\beta, u} \frac{\beta - 1}{\beta - 2}\right) \left(\frac{n}{CC_{\beta, u}} \left(\ln(D) + 1\right)\right) \\
        &= \frac{2(\beta - 1)}{C(\beta -2)} \cdot n (\ln(D) + 1).
    \end{align*}

\textbf{Case 3: $u \ge \sqrt{n}$.}

    In this case for all distances $d \in[1..n]$ we have $u > \sqrt{\frac{n}{d}}$, hence we should take the lower bound on $P_d$ from the right column of Table~\ref{tbl:pd}.

    \textbf{If} $\beta \le 1$, then we have
    \begin{align*}
        E[T_I] &\le \sum_{d = 1}^D P_d^{-1} \le \sum_{d = 1}^D\frac{3 - \beta}{CC_{\beta, u}} \cdot  u^{1 - \beta} = \frac{3 - \beta}{CC_{\beta, u}} \cdot Du^{1 - \beta}.
    \end{align*}

    \textbf{If} $\beta \in (1, 3)$, then by Lemma~\ref{lem:part-sum} we have
    \begin{align*}
        E[T_I] &\le \sum_{d = 1}^D P_d^{-1} \le \sum_{d = 1}^D \frac{2}{CC_{\beta, u}} \cdot  \sqrt{\frac{n}{d}}^{\beta - 1} =  \frac{2\sqrt{n}^{\beta - 1}}{CC_{\beta, u}}  \sum_{d = 1}^D d^{-\frac{\beta - 1}{2}} \\
        &\le \frac{2\sqrt{n}^{\beta - 1}}{CC_{\beta, u}} \cdot \left(\frac{D^{\frac{3 - \beta}{2}} - 1}{\left(\frac{3 - \beta}{2}\right)} + 1\right) = \frac{4\sqrt{n}^{\beta - 1}}{CC_{\beta, u}(3 - \beta)} \left(\sqrt{D}^{3 - \beta} - \frac{\beta - 1}{2}\right) \\
        &\le \frac{4\sqrt{n}^{\beta - 1}\sqrt{D}^{3 - \beta}}{CC_{\beta, u}(3 - \beta)}  
    \end{align*}

    \textbf{If} $\beta = 3$, then we have 
    \begin{align*}
        E[T_I] &\le \sum_{d = 1}^D P_d^{-1} \le \sum_{d = 1}^D \frac{1}{CC_{\beta, u}} \cdot \frac{n}{d\ln\left(\sqrt{\frac{n}{d}} + 1\right)}
    \end{align*}

    As we have shown in the second case of this proof in the part when $\beta = 3$, $d\ln(\sqrt{\frac{n}{d}} + 1)$ is an increasing function of $d$. Therefore, by Lemma~\ref{lem:sum-int} we have
    \begin{align*}
        E[T_I] &\le \frac{1}{CC_{\beta, u}} \left( \frac{n}{\ln(\sqrt{n} + 1)} + \int_1^D \frac{n dx}{x\ln\left(\sqrt{\frac{n}{x}} + 1\right)} \right).
    \end{align*}
    The integral in this expression is is the same as in the second case of this proof (in the part for $\beta = 3$) with $u = \sqrt{n}$, hence we estimate it as follows.
    \begin{align*}
        \int_1^D \frac{n dx}{x\ln\left(\sqrt{\frac{n}{x}} + 1\right)} \le \begin{cases}
            2n\ln\left(\frac{\ln(n)}{\ln(n/D)}\right), &\text{ if } D \le \frac{n}{e^2}, \\
            2n\left(\ln\ln\sqrt{n} + \frac{1}{\ln(2)}\right), &\text{ if } D > \frac{n}{e^2}.
        \end{cases}
    \end{align*}
    For the second case we slightly weaken the bound in order to improve the readability as follows.
    \begin{align*}
        2n\left(\ln\ln\sqrt{n} + \frac{1}{\ln(2)}\right) = 2n\left(\ln\ln(n) - ln(2) + \frac{1}{\ln(2)}\right) \le 2n(\ln\ln(n) + 1).
    \end{align*}
    Hence, for $D \le \frac{n}{e^2}$ we have
    \begin{align*}
        E[T_I] \le \frac{1}{CC_{\beta, u}} \left( \frac{n}{\ln(\sqrt{n} + 1)} + 2n\ln\left(\frac{\ln(n)}{\ln(n/D)}\right) \right).
    \end{align*}
    We note that for $D \ge 2$ this bound can be simplified to
    \begin{align*}
        E[T_I] \le \frac{4n}{CC_{\beta, u}} \ln\left(\frac{\ln(n)}{\ln(n/D)}\right),
    \end{align*}
    since we have
    \begin{align*}
        2n\ln\left(\frac{\ln(n)}{\ln(n/D)}\right) &\ge 2n\ln\left(\frac{\ln(n)}{\ln(n/2)}\right) = -2n \ln\left(\frac{\ln(n) - \ln(2)}{\ln(n)}\right) \\
        &= -2n \ln\left(1 - \frac{\ln(2)}{\ln(n)}\right) \ge 2n\frac{\ln(2)}{\ln(n)} \\
        &\ge \frac{n}{\ln(\sqrt{n})} \ge \frac{n}{\ln(\sqrt{n} + 1)}.
    \end{align*}
    
    For $D > \frac{n}{e^2}$ we have
    \begin{align*}
        E[T_I] \le \frac{1}{CC_{\beta, u}} \left( \frac{n}{\ln(\sqrt{n} + 1)} +  2n(\ln\ln(n) + 1) \right) \le \frac{3n(\ln\ln(n) + 1)}{CC_{\beta, u}}.
    \end{align*}

    \textbf{If} $\beta > 3$, then by Lemma~\ref{lem:part-sum} we have
    \begin{align*}
        E[T_I]&\le \sum_{d = 1}^D (P_d')^{-1} = \frac{n}{CC_{\beta, u}} \sum_{d = 1}^D \frac{1}{d} \le \frac{n}{CC_{\beta, u}} \left(\ln(D) + 1\right).
    \end{align*}

    Now we transform the estimates on the expectation of $T_I$ into the expectation of $T_F$ using the Wald's equation (Lemma~\ref{lem:wald}) and the estimates for $E[2\lambda]$ from Lemma~\ref{lem:expected-lambda}.

    \textbf{If} $\beta \le 1$, then we have
    \begin{align*}
        E[T_F] &= E[T_I]E[2\lambda] \le \left(2(3 - \beta)C_{\beta, u} \frac{u^{2 - \beta}}{2 - \beta}\right)  \left( \frac{3 - \beta}{CC_{\beta, u}} \cdot Du^{1 - \beta} \right) \\
        &= \frac{2(3 - \beta)^2}{C(2 - \beta)} \cdot Du.
    \end{align*}

    \textbf{If} $\beta \in (1, 2)$, then we have
    \begin{align*}
        E[T_F] &= E[T_I]E[2\lambda] \le \left(2(3 - \beta)C_{\beta, u} \frac{u^{2 - \beta}}{2 - \beta}\right) \cdot \left(\frac{4\sqrt{n}^{\beta - 1}\sqrt{D}^{3 - \beta}}{CC_{\beta, u}(3 - \beta)}\right) \\
        &= \frac{8}{C(2 - \beta)} \sqrt{n}^{\beta - 1}\sqrt{D}^{3 - \beta} u^{2 - \beta} = \frac{8}{C(2 - \beta)} \cdot \frac{n}{u}\left(u\sqrt{\frac{D}{n}}\right)^{3 - \beta}.
    \end{align*}

    \textbf{If} $\beta = 2$, then we have
    \begin{align*}
        E[T_F] &= E[T_I]E[2\lambda] \le \left(2C_{2, u} (\ln(u) + 1)\right) \cdot \left(\frac{4\sqrt{nD}}{CC_{\beta, u}}\right) \\
        &= \frac{8}{C} (\ln(u) + 1)\sqrt{nD}.
    \end{align*}

    \textbf{If} $\beta \in (2, 3)$, then we have
    \begin{align*}
        E[T_F] &= E[T_I]E[2\lambda] \le \left(2C_{\beta, u} \frac{\beta - 1}{\beta - 2}\right) \cdot \left(\frac{4\sqrt{n}^{\beta - 1}\sqrt{D}^{3 - \beta}}{CC_{\beta, u}(3 - \beta)}\right) \\
        &= \frac{8(\beta - 1)}{C(\beta - 2)(3 - \beta)} \cdot\sqrt{n}^{\beta - 1}\sqrt{D}^{3 - \beta}.
    \end{align*}

    \textbf{If} $\beta =3$ \textbf{and} $D \le \frac{n}{e^2}$, then we have
    \begin{align*}
        E[T_F] &= E[T_I]E[2\lambda] \le \left(2C_{\beta, u} \frac{\beta - 1}{\beta - 2}\right) \cdot \frac{1}{CC_{\beta, u}} \\
        &\cdot \left( \frac{n}{\ln(\sqrt{n} + 1)} + 2n\ln\left(\frac{\ln(n)}{\ln(n/D)}\right) \right)\\
        &= \frac{2(\beta - 1)}{C(\beta - 2)} \cdot\left( \frac{n}{\ln(\sqrt{n} + 1)} + 2n\ln\left(\frac{\ln(n)}{\ln(n/D)}\right) \right).
    \end{align*}

    \textbf{If} $\beta =3$ \textbf{and} $D > \frac{n}{e^2}$, then we have
    \begin{align*}
        E[T_F] &= E[T_I]E[2\lambda] \le \left(2C_{\beta, u} \frac{\beta - 1}{\beta - 2}\right) \cdot \left(\frac{3n(\ln\ln(n) + 1)}{CC_{\beta, u}}\right) \\
        &= \frac{6(\beta - 1)}{C(\beta - 2)} \cdot n(\ln\ln(n) + 1).
    \end{align*}

    \textbf{If} $\beta >3$, then we have
    \begin{align*}
        E[T_F] &\le \left(2C_{\beta, u} \frac{\beta - 1}{\beta - 2}\right) \left(\frac{n}{CC_{\beta, u}} \left(\ln(D) + 1\right)\right) \\
        &= \frac{2(\beta - 1)}{C(\beta -2)} \cdot n (\ln(D) + 1).
    \end{align*}

\end{proof}


\section{Black-box Complexity}\label{sec:bbc}

Now we turn to proving the black-box complexity bounds for the case of starting with an already good solution.
Informally, it is the same as the old black-box complexity, with a restriction on the employed algorithms if necessary,
but the algorithm is additionally supplied with a search point which is supposed to be a good solution.
As we are interested in the way the complexity of the problem scales depending on how such a point is generated with respect to the particular problem instance,
we introduce an oracle function that creates these points depending on the problem instance.

\begin{definition}
    Let $\mathcal{P} = (\mathcal{S}, \mathcal{I}, \mathcal{O})$ be the problem class consisting of a search space $\mathcal{P}$,
    a set of problem instances $\mathcal{I}$ where each problem instance is a function from $\mathcal{S}$ to $\R$,
    and an oracle function $\mathcal{O} : \mathcal{I} \to \mathcal{S}$ that generates an initial search point.
    Let $T_{A}(I, q_0)$ be the expected time that an algorithm $A$ needs to solve a problem instance $I \in \mathcal{I}$
    given the initial search point $q_0 \in \mathcal{P}$.
    The black-box complexity with an oracle of $\mathcal{P}$ for a class of algorithms $\mathcal{A}$ is $\min_{A \in \mathcal{A}} \max_{I \in \mathcal{I}} T_{A}(I, \mathcal{O}(I))$.
\end{definition}

Now we prove the main result of this section.

\begin{theorem}
    The unrestricted black-box complexity of \onemax of a large enough size $n$ with an oracle function giving a point $x^{(0)}$ with a Hamming distance of $D$ to the optimum is
    \begin{align*}
        BBC_{\om}(n, D) &\ge \left\lceil \log_{1 + \min\{D, n-D\}} \binom{n}{D} \right\rceil - 1\\
        BBC_{\om}(n, D) &\le \left(1 + \frac{1}{n}\right) \cdot 2 \log_{1 + \min\{D, n-D\}} \binom{n}{D}.
    \end{align*}
    and thus, asymptotically,
    \begin{equation*}
        BBC_{\om}(n, D) = \Theta\left(\log_{1 + \min\{D, n-D\}} \binom{n}{D}\right),
    \end{equation*}
    which for $1 \le D \le n/2$ simplifies to
    \begin{equation*}
        BBC_{\om}(n, D) = \Theta\left(\frac{D \log (n/D)}{\log (1 + D)}\right).
    \end{equation*}
\end{theorem}

\begin{proof}
The last statement essentially follows from the previous one using the well-known bound $(n/D)^D \le \binom{n}{D} \le (en/D)^D$:
\begin{equation*}
    \log_{1+D} \binom{n}{D} = \frac{\log \binom{n}{D}}{\log (1+D)} = \frac{D \log (\Theta(n/D))}{\log (1+D)} = \Theta\left(\frac{D \log (n/D)}{\log (1 + D)}\right).
\end{equation*}
Now we turn to proving the actual bounds on the black-box complexity.

To prove the \textbf{lower bound}, we note that if $x^{(0)}$ is known to be different from the optimum $x^*$ in $D$ bits, there are only $\binom{n}{D}$ points which can be the optimum,
so in terms of Lemma~\ref{lem:jansen} we have $|S| = \binom{n}{D}$. It remains to determine how many possible fitness values can there be for an arbitrary query $x$.

Let $d^{(0)}$ be the Hamming distance between $x$ and $x^{(0)}$. If there is an upper bound $k$ on the number of possible fitness values of $x$ that holds for each $d^{(0)}$,
then we can use this $k$ in Lemma~\ref{lem:jansen} even though the size of the set of all possible fitness values of $x$ over all $d^{(0)}$ is larger.
Let $d_1$ be the number of bits that are different in $x$ and $x^*$ but identical in $x$ and $x^{(0)}$. Let $d_2$ be the number of bits that are different in $x$ and $x^*$ and also different in $x$ and $x^{(0)}$. Then there are the following relations:
\begin{equation*}
    f(x) = n - (d_1 + d_2); \quad
    d^{(0)} = D + d_2 - d_1; \quad
    0 \le d_1 \le n - D; \quad
    0 \le d_2 \le D.
\end{equation*}

Since $d_2 - d_1$ is fixed by fixing $d^{(0)}$, they can change only synchronously, so, because the number of possible values of $d_1$ is limited by $n-D+1$ and the number of possible values of $d_2$
is limited by $D+1$, the number of allowed pairs $(d_1, d_2)$ is at most $1 + \min\{D, n - D\}$, so is the number of possible values of $f(x)$.
Hence we have $k = 1 + \min\{D, n - D\}$, and altogether the lower bound is, per Lemma~\ref{lem:jansen}
\begin{equation*}
    BBC_{\om}(n, D) \ge \left\lceil \log_{1 + \min\{D, n-D\}} \binom{n}{D} \right\rceil - 1.
\end{equation*}

To prove the \textbf{upper bound}, we use the classic idea of performing enough random queries so that, with probability at most $1 - \frac{1}{n}$, there is only one optimum that agrees with all these queries~\cite{erd63}.

We consider individual search points first in the same way as in~\cite{DoerrJKLWW11}.
If a point $x_d$ is at a distance $d = 2h$ from the optimum $x^*$, a random query has the same Hamming distance to both $x^*$ and $x_d$, and hence agrees with both,
if and only if:
\begin{itemize}
    \item in those $d$ bits which are different between $x^*$ and $x_d$, it differs from $x^*$ in exactly $h$ bits;
    \item in other $n - d$ bits, any combination is acceptable.
\end{itemize}
Hence, out of $2^{2h}$ combinations of these $d$ bits, only $\binom{2h}{h}$ do not invalidate $x_d$ as a potential optimum. The probability of a single random query to agree with both $x^*$ and $x_d$ is thus $\binom{2h}{h} \cdot 2^{-2h}$.
The probability of a point $x_d$ to ``survive'' for $t$ independent queries is $(\binom{2h}{h} \cdot 2^{-2h})^t$.

Next, we make use of the existence of the given point $x^{(0)}$, which already filters out all search points except for $\binom{n}{D}$ ones. Any potential optimum shall be at the Hamming distance $D$
from $x^{(0)}$, so the $d=2h$ bits that differentiate that optimum from $x^*$ should be equally distributed among those $D$ bits in which $x^*$ and $x^{(0)}$ coincide, and among the other $n-D$ bits.
Hence, there are $\binom{D}{h} \binom{n-D}{h}$ potential optima that are located at a distance of $2h$ from $x^*$.

We consider now, for each half-distance $h$, a union bound for the probability $p_h$ that any potential optimum at a distance $2h$ from $x^*$ still agrees with $t$ random queries:
\begin{equation}
    p_h \le \binom{D}{h} \binom{n-D}{h} \left(\binom{2h}{h} \cdot 2^{-2h}\right)^t. \label{eq:halfd}
\end{equation}

Our goal is to choose $t$ so that all $p_h$ are small enough, and, in particular, so that $\sum_{h=1}^{D} p_h \le 1/n$.
We choose $t$ to be $2 \log_{1 + \min\{D, n-D\}} \binom{n}{D}$, that is, twice as large as the lower bound.
Since all involved expressions are symmetric with respect to swapping $D$ and $n-D$, it is enough to consider only $D \le n/2$. 

We consider four cases.

\textbf{Case 1}: $D=1$. In this case, only $h=1$ is possible, and the resulting bound is
\begin{equation*}
    \sum_{h=1}^{D} p_h = p_1 \le \binom{1}{1} \binom{n-1}{1} \left(\binom{2}{1} \cdot 2^{-2}\right)^{2 \log_2 \binom{n}{1}} = \frac{n - 1}{n^2} \le \frac{1}{n}.
\end{equation*}

For larger values of $D$, we simplify~\eqref{eq:halfd} first. It follows from Stirling's formula that $\binom{2h}{h} \cdot 2^{-2h} \le \sqrt{1 / (\pi h)}$,
so we can rewrite the bound as follows.
\begin{equation}
    p_h \le \binom{D}{h} \binom{n-D}{h} (\pi h)^{-\log_{1 + D} \binom{n}{D}} = \binom{D}{h} \binom{n-D}{h} \binom{n}{D}^{-\frac{\ln (\pi h)}{\ln (1 + D)}}_{.} \label{eq:ph}
\end{equation}

\textbf{Case 2}: $D=\Theta(1)$. In this case, $\binom{n}{D} = \Theta(n^D)$, $\binom{D}{h} = \Theta(1)$, $\binom{n - D}{h} = \Theta(n^h)$, and from~\eqref{eq:ph} we obtain
\begin{equation*}
    p_h \le \Theta(1) \Theta\left(n^{h}\right) \Theta\left(n^{-D \frac{\ln (\pi h)}{\ln (1 + D)}}\right) = \Theta\left(n^{h - D \frac{\ln (\pi h)}{\ln (1 + D)}}\right).
\end{equation*}

The power function, $Q(D, h) = h - D \frac{\ln (\pi h)}{\ln (1 + D)}$, has $Q(D, h)'_h = 1 - \frac{D}{\ln (1+D)} \frac{1}{h}$, which grows monotonically with $h$ for $h \ge 1$. This means that with a fixed $D$ the function $Q(D, h)$ is convex downwards, and its maximum for $h \in [1..D]$ is either at $h = 1$ or at $h = D$. We now show that both $Q(D, 1)$ and $Q(D, D)$ are less than one.
\begin{itemize}
    \item $Q(D, 1) = 1 - \frac{D}{\ln (1+D)} \ln \pi$ has the following derivative:
    \begin{equation*}
        Q(D, 1)' = -\ln \pi \cdot \frac{\ln (1+D) - \frac{D}{1+D}}{(\ln (1+D))^2},
    \end{equation*}
    which is negative for $D \ge 2$ as $\ln (1+D) \ge \ln 3 > 1 > \frac{D}{1+D}$.
    Hence, $Q(D,1)$ is decreasing and $Q(D,1) \le Q(2, 1)$, which can be computed and shown to be less than $-1.08$.
    \item $Q(D, D) = D - D \frac{\ln (\pi D)}{\ln (1+D)}$ has the following derivative:
    \begin{equation*}
        Q(D, D)' = \frac{(\ln (1+D) - 1) (\ln (1+D) - \ln (\pi D)) - \frac{1}{D+1} \ln (\pi D)}{(\ln (1+D))^2},
    \end{equation*}
    which is negative for $D \ge 2$ as $\ln (1+D) \ge \ln 3 > 1.09$ and $\ln (\pi D) > \ln (1+D)$.
    Hence, $Q(D,D)$ is decreasing and $Q(D, D) < Q(2, 2)$, which can be computed and shown to be less than $-1.34$.
\end{itemize}
As a result, $p_h \le O(n^{-1.08})$, and a sum of $D+1 = \Theta(1)$ such values is $O(n^{-1.08}) < \frac{1}{n}$ for large enough $n$.

\textbf{Case 3}: $D = \omega(1), h \ge 0.4 D$. In this case, we further simplify~\eqref{eq:ph} using Lemma~\ref{lem:vm-corollary}:
\begin{equation*}
    p_h \le \binom{D}{h} \binom{n-D}{h} \binom{n}{D}^{-\frac{\ln (\pi h)}{\ln (1 + D)}} \le \binom{n}{D} \cdot \binom{n}{D}^{-\frac{\ln (\pi h)}{\ln (1 + D)}} = \binom{n}{D}^{1-\frac{\ln (\pi h)}{\ln (1 + D)}}_{.}
\end{equation*}

Note that $\frac{\ln (\pi h)}{\ln (1+D)} \ge \frac{\ln (0.4 \pi D)}{\ln (1+D)} > \frac{\ln (1.25 (1+D))}{\ln (1+D)} = 1 + \frac{\ln 1.25}{\ln (1+D)} > 1 + \frac{0.22}{\ln (1+D)}$ for sufficiently large $D$.
Using that, we continue as follows:
\begin{equation*}
    p_h \le \binom{n}{D}^{1-\frac{\ln (\pi h)}{\ln (1 + D)}} \le \left(\frac{n}{D}\right)^{-\frac{0.22D}{\ln (1+D)}}
        \le \begin{cases}
                \sqrt{n}^{-\omega(1)} &\text{if } \omega(1) \le D \le \sqrt{n}, \\
                2^{-\Omega(\sqrt{n} / \ln{n})} &\text{if } \sqrt{n} \le D \le n / 2,
            \end{cases}
\end{equation*}
which in both cases does not exceed $n^{-\omega(1)}$. A sum of at most $D + 1 = O(n)$ such values is also $n^{-\omega(1)}$, which is less than $1/n$ for sufficiently large $n$.

\textbf{Case 4}: $D = \omega(1), 1 \le h \le 0.4 D$. We shall prove that $\binom{D}{h}^{1 / \ln (\pi h)}$ increases with $h$ for all sufficiently large $D$ as long as $h \le 0.4 D$.
Once that is done, we start from~\eqref{eq:ph} as follows:
\begin{align*}
    p_h &\le \binom{D}{h} \binom{n-D}{h} \binom{n}{D}^{-\frac{\ln (\pi h)}{\ln (1 + D)}} \\
        &= \left(\binom{D}{h}^{\frac{1}{\ln (\pi h)}} \binom{n-D}{h}^{\frac{1}{\ln (\pi h)}} \binom{n}{D}^{-\frac{1}{\ln (1 + D)}}\right)^{\ln (\pi h)} \\
        &\le \left(\binom{D}{0.4D}^{\frac{1}{\ln (0.4 \pi D)}} \binom{n-D}{0.4D}^{\frac{1}{\ln (0.4 \pi D)}} \binom{n}{D}^{-\frac{1}{\ln (1 + D)}}\right)^{\ln (\pi h)} \\
        &= \left(\binom{D}{0.4D} \binom{n-D}{0.4D} \binom{n}{D}^{-\frac{\ln (0.4 \pi D)}{\ln (1 + D)}}\right)^{\frac{\ln (\pi h)}{\ln (0.4 \pi D)}}_{,}
\end{align*}
where the second inequality applies the statement above for both $\binom{D}{h}$ and $\binom{n-D}{h}$. Next we apply Case~3 almost entirely to the expression in parentheses (with a substitution $h = 0.4D$),
which results in
\begin{align*}
    p_h &\le \left( \left(\frac{n}{D}\right)^{\frac{-0.22D}{\ln (1+D)}} \right)^{\frac{\ln (\pi h)}{\ln (0.4 \pi D)}}
         \le \left( \left(\frac{n}{D}\right)^{\frac{-0.22D}{\ln (1+D)}} \right)^{\frac{\ln (\pi)}{\ln (0.4 \pi D)}} \\
        &= \left(\frac{n}{D}\right)^{-\Theta\left(\frac{D}{(\ln D)^2}\right)}
        \le \begin{cases}
                \sqrt{n}^{-\omega(1)} &\text{if } \omega(1) \le D \le \sqrt{n}, \\
                2^{-\Omega(\sqrt{n} / (\ln{n})^2)} &\text{if } \sqrt{n} \le D \le n / 2,
            \end{cases}
\end{align*}
which in both cases does not exceed $n^{-\omega(1)}$. A sum of at most $D + 1 = O(n)$ such values is also $n^{-\omega(1)}$, which is less than $1/n$ for sufficiently large $n$.

The only remaining thing is to prove that $\binom{D}{h}^{1 / \ln (\pi h)}$ increases with $h$ for all sufficiently large $D$ as long as $h \le 0.4 D$.
To do that, we consider a logarithm of the ratio of the consecutive values, which we want to be greater than zero.
\begin{align*}
    \ln \frac{\binom{D}{h+1}^{\frac{1}{\ln (\pi (h+1))}}}{\binom{D}{h}^{\frac{1}{\ln (\pi h)}}}
        &= \frac{\ln \binom{D}{h+1}}{\ln (\pi (h+1))} - \frac{\ln \binom{D}{h}}{\ln (\pi h)}
         = \frac{\ln \binom{D}{h} + \ln \frac{D-h}{h+1}}{\ln (\pi (h+1))} - \frac{\ln \binom{D}{h}}{\ln (\pi h)} \\
        &= \frac{\ln \frac{D-h}{h+1}}{\ln (\pi (h+1))} - \ln \binom{D}{h} \cdot \left( \frac{1}{\ln (\pi h)} - \frac{1}{\ln (\pi (h+1))} \right) \\
        &= \frac{\ln \frac{D-h}{h+1}}{\ln (\pi (h+1))} - \ln \binom{D}{h} \cdot \left( \frac{\ln (\pi (h+1)) - \ln (\pi h)}{\ln (\pi h) \ln (\pi (h+1))} \right) \\
        &= \frac{\ln \frac{D-h}{h+1} - \frac{\ln \binom{D}{h}}{\ln (\pi h)} (\ln (h+1) - \ln (h)) }{\ln (\pi (h+1))}\\
        &\ge \frac{\ln \frac{D-h}{h+1} - \ln \frac{eD}{h} \cdot \frac{h}{\ln (\pi h)} (\ln (h+1) - \ln (h)) }{\ln (\pi (h+1))}.
\end{align*}

We now show that $R(h) = \frac{h (\ln (h+1) - \ln (h))}{\ln (\pi h)}$ is a decreasing function for $h \ge 1$.
The derivative of the numerator $h (\ln (h+1) - \ln (h)) = -\sum_{i=1}^{\infty} \frac{(-1)^i}{i h^i}$ is $\sum_{i=2}^{\infty} \frac{(i - 1) \cdot (-1)^i}{i h^i}$,
which is positive and less than $\frac{1}{2h^2}$. The derivative of the denominator $\ln (\pi h)$ is $1/h$, which is at least two times greater.
As $R(1) = \frac{\ln 2}{\ln \pi} \in (0.6, 0.61)$, $R(h)'$ is negative.

It remains to show that $\ln \frac{D-h}{h+1} - \ln \frac{eD}{h} \cdot R(h)$ is positive for sufficiently large $D$. We consider two cases:
\begin{itemize}
    \item $1 \le h \le 0.02 D$. We use $R(h) < 0.61$ and continue as follows:
    \begin{align*}
        \ln \frac{D-h}{h+1} &- \ln \frac{eD}{h} \cdot R(h)
            \ge \ln\frac{0.98D}{h+1} - 0.61\ln \frac{eD}{h}
         \\&= \ln (0.98) + \ln (D) - \ln (h+1) - 0.61 (1 + \ln (D)) + 0.61 \ln (h)
         \\&= 0.39 (\ln (D) - \ln (h)) - (\ln (h+1) - \ln (h)) + \ln (0.98) - 0.61
         \\&\ge -0.39 \ln (0.02) - \ln (2) + \ln (0.98) - 0.61 > 0.2,
    \end{align*}
    where the last line uses $\ln (h+1) - \ln (h) \ge \ln 2$.
    \item $0.02 D \le h \le 0.4 D$. Here we note that $h = \omega(1)$ as $D = \omega(1)$,
          so we can use that 
          $R(h) = \frac{h \ln (1 + \frac{1}{h})}{\ln (\pi h)} = \Theta(\frac{1}{\ln h}) = o(1)$.
          This results in
    \begin{align*}
        \ln \frac{D-h}{h+1} &- \ln \frac{eD}{h} \cdot R(h)
            \ge \ln\frac{0.6D}{0.4D} - \ln \frac{eD}{0.02D} \cdot o(1)
            \ge 0.4 - 5 \cdot o(1),
    \end{align*}
    which is positive for sufficiently large $n$ (and, as a result, sufficiently large $D$ and $h$).
\end{itemize}

This finishes the proof for Case~4.
As the proven four cases cover the entire range of variables, the entire theorem is now proven.
\end{proof}

\section{Experiments}
\label{sec:experiments}
\newcommand{\thefontsize}{\scriptsize}

To highlight that the theoretically proven behavior of the algorithms is not strongly affected by the constants hidden in the asymptotic notation, we conducted experiments with the following settings:
\begin{itemize}
    \item fast \ollga with $\beta\in\{2.1, 2.3, 2.5, 2.7, 2.9\}$ and the upper limit $u=n/2$;
    \item self-adjusting \ollga, both in its original uncapped form and with $\lambda$ capped from above by $2 \log (n+1)$ as proposed in~\cite{BuzdalovD17};
    \item the mutation-only algorithms $(1+1)$~EA and RLS.
\end{itemize}
In all our experiments, the runtimes are averaged over 100 runs, unless said otherwise.

In Figure~\ref{exp:sqrt} we show the mean running times of these algorithms 
when they start in Hamming distance roughly $\sqrt{n}$ from the optimum.
For this experiment, to avoid possible strange effects from particular numbers, we used a different initialization for all algorithms,
namely that in the initial individual every bit was set to $0$ with probability $\frac{1}{\sqrt{n}}$ and it was set to $1$ otherwise. As the figure shows, all algorithms with a heavy-tailed choice of $\lambda$ outperformed the mutation-based algorithms,
which struggled from the coupon-collector effect.

\begin{figure}[!t]
\centering
\begin{tikzpicture}
\begin{axis}[width=0.9\linewidth, height=0.3\textheight, xmode=log, log base x=2, ymode=log, grid=major,
             xlabel={Problem size $n$}, ylabel={Evaluations / $\sqrt{nD}$},
             legend pos=north west, legend columns=2, cycle list name=myplotcycle,
             every axis plot/.append style={very thick}]
\addplot plot [error bars/.cd, y dir=both, y explicit] coordinates {(32,8.93)+-(0,3.734)(64,10.41)+-(0,4.181)(128,11.62)+-(0,3.314)(256,12.65)+-(0,3.046)(512,13.97)+-(0,3.015)(1024,14.31)+-(0,3.209)(2048,16.07)+-(0,3.419)(4096,17.05)+-(0,2.757)(8192,17.99)+-(0,2.703)(16384,20.74)+-(0,4.063)(32768,22.67)+-(0,3.284)(65536,24.81)+-(0,4.18)(131072,28.08)+-(0,5.315)(262144,33.18)+-(0,5.724)(524288,36.16)+-(0,5.957)(1048576,42.28)+-(0,6)(2097152,47.83)+-(0,6.535)(4194304,58.62)+-(0,9.278)};
\addlegendentry{\thefontsize $\lambda\in [1..2\ln (n + 1)]$};
\addplot plot [error bars/.cd, y dir=both, y explicit] coordinates {(32,8.854)+-(0,3.627)(64,9.281)+-(0,3.512)(128,11.98)+-(0,3.499)(256,11.84)+-(0,3.104)(512,13.33)+-(0,2.725)(1024,13.91)+-(0,2.035)(2048,14.35)+-(0,1.922)(4096,14.51)+-(0,2.061)(8192,14.77)+-(0,1.738)(16384,14.68)+-(0,1.3)(32768,15.03)+-(0,1.031)(65536,15.35)+-(0,1.071)(131072,15.26)+-(0,0.759)(262144,15.62)+-(0,0.7812)(524288,15.72)+-(0,0.6158)(1048576,15.72)+-(0,0.5515)(2097152,15.81)+-(0,0.5271)(4194304,15.84)+-(0,0.4097)};
\addlegendentry{\thefontsize $\lambda\in [1..n]$};
\addplot plot [error bars/.cd, y dir=both, y explicit] coordinates {(32,10.98)+-(0,5.743)(64,12.79)+-(0,5.424)(128,15.59)+-(0,5.613)(256,17.97)+-(0,5.377)(512,22.26)+-(0,6.79)(1024,23.79)+-(0,5.906)(2048,29.03)+-(0,7.752)(4096,31.28)+-(0,8.704)(8192,33.49)+-(0,7.757)(16384,35.27)+-(0,7.504)(32768,39.5)+-(0,8.207)(65536,42.52)+-(0,8.795)(131072,44.64)+-(0,8.427)(262144,48.64)+-(0,12.37)(524288,52.74)+-(0,13.09)(1048576,54.77)+-(0,13.4)(2097152,53.88)+-(0,7.8)(4194304,58.9)+-(0,13.46)};
\addlegendentry{\thefontsize $\lambda\sim\text{pow}(2.1)$};
\addplot plot [error bars/.cd, y dir=both, y explicit] coordinates {(32,10.46)+-(0,4.797)(64,12.09)+-(0,5.409)(128,15.09)+-(0,5.276)(256,17.76)+-(0,5.539)(512,20.71)+-(0,5.127)(1024,23.14)+-(0,5.499)(2048,26.07)+-(0,5.391)(4096,28.68)+-(0,5.455)(8192,32.76)+-(0,6.147)(16384,35.43)+-(0,5.789)(32768,37.77)+-(0,5.368)(65536,41.11)+-(0,6.503)(131072,44)+-(0,6.201)(262144,47.92)+-(0,5.266)(524288,52.01)+-(0,6.408)(1048576,55.28)+-(0,6.346)(2097152,58.84)+-(0,8.104)(4194304,62.01)+-(0,3.476)};
\addlegendentry{\thefontsize $\lambda\sim\text{pow}(2.3)$};
\addplot plot [error bars/.cd, y dir=both, y explicit] coordinates {(32,9.38)+-(0,4.865)(64,11.8)+-(0,5.368)(128,14.38)+-(0,6.239)(256,17.22)+-(0,5.261)(512,20.41)+-(0,6.149)(1024,24.36)+-(0,5.727)(2048,27.74)+-(0,5.764)(4096,30.98)+-(0,6.063)(8192,36.15)+-(0,6.316)(16384,39.27)+-(0,6.596)(32768,44.49)+-(0,6.861)(65536,50.12)+-(0,6.268)(131072,55.59)+-(0,7.039)(262144,61.49)+-(0,6.843)(524288,67)+-(0,5.776)(1048576,75.23)+-(0,7.042)(2097152,81.18)+-(0,5.673)(4194304,90.45)+-(0,6.226)};
\addlegendentry{\thefontsize $\lambda\sim\text{pow}(2.5)$};
\addplot plot [error bars/.cd, y dir=both, y explicit] coordinates {(32,10.02)+-(0,4.859)(64,11.51)+-(0,4.627)(128,14.07)+-(0,6.508)(256,17.93)+-(0,5.773)(512,20.65)+-(0,6.861)(1024,25.94)+-(0,6.4)(2048,29.35)+-(0,7.228)(4096,35.02)+-(0,7.85)(8192,40.58)+-(0,7.254)(16384,46.56)+-(0,8.224)(32768,54.08)+-(0,8.153)(65536,63.87)+-(0,10.39)(131072,70.85)+-(0,8.304)(262144,80.76)+-(0,10.28)(524288,94.19)+-(0,9.941)(1048576,109.1)+-(0,11.06)(2097152,124.7)+-(0,12.39)(4194304,142.4)+-(0,11.53)};
\addlegendentry{\thefontsize $\lambda\sim\text{pow}(2.7)$};
\addplot plot [error bars/.cd, y dir=both, y explicit] coordinates {(32,9.932)+-(0,5.114)(64,11.44)+-(0,4.734)(128,15.36)+-(0,5.551)(256,18.59)+-(0,6.023)(512,22.23)+-(0,6.051)(1024,28.18)+-(0,7.805)(2048,34.47)+-(0,8.684)(4096,38.63)+-(0,8.316)(8192,47.78)+-(0,9.093)(16384,57.22)+-(0,9.186)(32768,66.39)+-(0,10.45)(65536,79.58)+-(0,14.22)(131072,97.39)+-(0,14.85)(262144,110.5)+-(0,15.64)(524288,137.5)+-(0,26.78)(1048576,159.9)+-(0,21.54)(2097152,183.3)+-(0,22.8)(4194304,224)+-(0,29.1)};
\addlegendentry{\thefontsize $\lambda\sim\text{pow}(2.9)$};
\addplot plot [error bars/.cd, y dir=both, y explicit] coordinates {(32,10.06)+-(0,5.607)(64,12.05)+-(0,4.94)(128,17.3)+-(0,7.69)(256,22.31)+-(0,8.722)(512,29.87)+-(0,9.323)(1024,39.85)+-(0,12.23)(2048,50.14)+-(0,13.89)(4096,66.38)+-(0,17.8)(8192,83.17)+-(0,20.81)(16384,106.9)+-(0,21.99)(32768,133)+-(0,31.26)(65536,164.7)+-(0,34.79)(131072,209.1)+-(0,44.49)(262144,271)+-(0,57.31)(524288,342)+-(0,62.12)(1048576,419.8)+-(0,71.3)(2097152,520.5)+-(0,89.61)(4194304,630.9)+-(0,90.47)};
\addlegendentry{\thefontsize (1+1) EA};
\addplot plot [error bars/.cd, y dir=both, y explicit] coordinates {(32,6.145)+-(0,4.1)(64,8.104)+-(0,3.763)(128,10.37)+-(0,4.039)(256,12.62)+-(0,4.653)(512,17.77)+-(0,5.886)(1024,22.38)+-(0,6.094)(2048,30.56)+-(0,8.864)(4096,37.46)+-(0,9.897)(8192,47.67)+-(0,12.52)(16384,61.24)+-(0,13.9)(32768,75.16)+-(0,13.07)(65536,95.12)+-(0,18.83)(131072,122.1)+-(0,22.09)(262144,155.5)+-(0,25.34)(524288,185.8)+-(0,28.89)(1048576,245.3)+-(0,43.6)(2097152,296.9)+-(0,46.55)(4194304,368.8)+-(0,59.21)};
\addlegendentry{\thefontsize RLS};
\end{axis}
\end{tikzpicture}
\caption{Mean runtimes and their standard deviation of different algorithms on \onemax with initial Hamming distance $D$ from the optimum equal to $\sqrt{n}$ in expectation.
By $\lambda \in [1..u]$ we denote the self-adjusting parameter choice via the one-fifth rule in the interval $[1..u]$. 
The indicated confidence interval for each value $X$ is $[E[X] - \sigma(X), E[x] + \sigma(X)]$, where $\sigma(X)$ is the standard deviation of~$X$.
The runtime is normalized by $\sqrt{nD}$, so that the plot of the self-adjusting \ollga is a horizontal line.}
\label{exp:sqrt}
\begin{tikzpicture}
\begin{axis}[width=0.9\linewidth, height=0.3\textheight, xmode=log, log base x=2, ymode=log, grid=major,
             xlabel={Problem size $n$}, ylabel={Evaluations / $\sqrt{nD}$},
             legend pos=north west, legend columns=2, cycle list name=myplotcycle,
             every axis plot/.append style={very thick}]
\addplot plot [error bars/.cd, y dir=both, y explicit] coordinates {(32,8.254)+-(0,5.073)(64,10.09)+-(0,4.979)(128,11.53)+-(0,4.668)(256,12.39)+-(0,4.987)(512,14.09)+-(0,5.716)(1024,14.63)+-(0,5.695)(2048,16.24)+-(0,6.031)(4096,22.01)+-(0,8.861)(8192,24.76)+-(0,10.7)(16384,29.38)+-(0,12.19)(32768,35.17)+-(0,14.05)(65536,51.25)+-(0,23.72)(131072,61.92)+-(0,25.04)(262144,89.11)+-(0,38.45)(524288,108.3)+-(0,39.1)(1048576,142.1)+-(0,47.31)(2097152,198)+-(0,76.98)(4194304,266.1)+-(0,108.9)};
\addlegendentry{\thefontsize $\lambda\in [1..2\ln (n + 1)]$};
\addplot plot [error bars/.cd, y dir=both, y explicit] coordinates {(32,9.17)+-(0,5.222)(64,10.75)+-(0,4.612)(128,10.24)+-(0,4.902)(256,12.67)+-(0,4.506)(512,12.72)+-(0,5.397)(1024,12.52)+-(0,4.938)(2048,13.5)+-(0,5.371)(4096,13.66)+-(0,4.734)(8192,13.53)+-(0,4.087)(16384,14.06)+-(0,4.051)(32768,13.19)+-(0,4.094)(65536,13.69)+-(0,4.445)(131072,14.69)+-(0,4.173)(262144,13.85)+-(0,4.131)(524288,14.05)+-(0,3.453)(1048576,14.1)+-(0,3.696)(2097152,14.45)+-(0,3.397)(4194304,13.83)+-(0,3.295)};
\addlegendentry{\thefontsize $\lambda\in [1..n]$};
\addplot plot [error bars/.cd, y dir=both, y explicit] coordinates {(32,9.618)+-(0,5.059)(64,11.85)+-(0,7.875)(128,15.78)+-(0,8.51)(256,18.61)+-(0,10.03)(512,21.4)+-(0,11.67)(1024,23.84)+-(0,11.63)(2048,25.04)+-(0,11.97)(4096,28.09)+-(0,14.57)(8192,32.81)+-(0,15.77)(16384,36.82)+-(0,19.73)(32768,37.63)+-(0,15.34)(65536,40.41)+-(0,19.79)(131072,46.01)+-(0,27.08)(262144,50.14)+-(0,30.89)(524288,47.94)+-(0,18.8)(1048576,61.05)+-(0,34.68)(2097152,51.59)+-(0,23.93)(4194304,53.97)+-(0,22.3)};
\addlegendentry{\thefontsize $\lambda\sim\text{pow}(2.1)$};
\addplot plot [error bars/.cd, y dir=both, y explicit] coordinates {(32,9.335)+-(0,4.837)(64,12.09)+-(0,6.642)(128,15.94)+-(0,8.376)(256,16.6)+-(0,6.689)(512,18.83)+-(0,8.611)(1024,23.41)+-(0,9.5)(2048,30.81)+-(0,12.91)(4096,29.89)+-(0,13.62)(8192,36.66)+-(0,17.51)(16384,41.7)+-(0,16.53)(32768,43.87)+-(0,18.65)(65536,52.31)+-(0,19.93)(131072,52.73)+-(0,19.84)(262144,63.61)+-(0,24.78)(524288,72.05)+-(0,31.18)(1048576,76.72)+-(0,29.18)(2097152,86.64)+-(0,31.04)(4194304,99.41)+-(0,39.47)};
\addlegendentry{\thefontsize $\lambda\sim\text{pow}(2.3)$};
\addplot plot [error bars/.cd, y dir=both, y explicit] coordinates {(32,8.423)+-(0,6.111)(64,11.7)+-(0,6.335)(128,15.3)+-(0,8.09)(256,17.73)+-(0,10.45)(512,21.81)+-(0,10.65)(1024,25.43)+-(0,13.23)(2048,29.88)+-(0,12.52)(4096,36.5)+-(0,13.54)(8192,42.51)+-(0,16.61)(16384,53.98)+-(0,21.43)(32768,57.55)+-(0,25.66)(65536,73.81)+-(0,26.59)(131072,84.26)+-(0,29.52)(262144,95.22)+-(0,36.46)(524288,117.2)+-(0,43.31)(1048576,133.5)+-(0,55.07)(2097152,168.4)+-(0,61.65)(4194304,194.1)+-(0,63.62)};
\addlegendentry{\thefontsize $\lambda\sim\text{pow}(2.5)$};
\addplot plot [error bars/.cd, y dir=both, y explicit] coordinates {(32,9.596)+-(0,5.456)(64,11.73)+-(0,7.415)(128,14.49)+-(0,7.521)(256,17.36)+-(0,8.622)(512,20.08)+-(0,10.91)(1024,31.29)+-(0,15.03)(2048,38.62)+-(0,18.59)(4096,45.13)+-(0,17.94)(8192,54.46)+-(0,21.05)(16384,66.84)+-(0,25.88)(32768,88.07)+-(0,31.05)(65536,119)+-(0,52.73)(131072,131)+-(0,50.57)(262144,156.3)+-(0,60.61)(524288,220.9)+-(0,85.1)(1048576,273.7)+-(0,98.87)(2097152,321.9)+-(0,120)(4194304,415.6)+-(0,134.1)};
\addlegendentry{\thefontsize $\lambda\sim\text{pow}(2.7)$};
\addplot plot [error bars/.cd, y dir=both, y explicit] coordinates {(32,9.053)+-(0,6.763)(64,11.93)+-(0,7.688)(128,15.93)+-(0,8.548)(256,20.1)+-(0,10.17)(512,26.96)+-(0,11.78)(1024,33.39)+-(0,17.82)(2048,40.29)+-(0,18.48)(4096,57.69)+-(0,22.65)(8192,67.86)+-(0,30.6)(16384,90.5)+-(0,40.47)(32768,109.8)+-(0,44.69)(65536,151.8)+-(0,63.41)(131072,202.8)+-(0,74.75)(262144,263)+-(0,112.5)(524288,330.3)+-(0,135.7)(1048576,458.6)+-(0,177)(2097152,596.2)+-(0,204.9)(4194304,790.3)+-(0,299.8)};
\addlegendentry{\thefontsize $\lambda\sim\text{pow}(2.9)$};
\addplot plot [error bars/.cd, y dir=both, y explicit] coordinates {(32,9.434)+-(0,5.99)(64,12.15)+-(0,8.613)(128,19.02)+-(0,10.25)(256,27.85)+-(0,14.66)(512,38.26)+-(0,18.33)(1024,54.2)+-(0,25.04)(2048,72.54)+-(0,33.45)(4096,99.36)+-(0,56.35)(8192,147.5)+-(0,61.82)(16384,207)+-(0,90.13)(32768,256.2)+-(0,102.5)(65536,417.3)+-(0,172.3)(131072,554.5)+-(0,248)(262144,736.8)+-(0,261.4)(524288,1093)+-(0,475.3)(1048576,1450)+-(0,570.8)(2097152,2163)+-(0,1093)(4194304,3214)+-(0,1539)};
\addlegendentry{\thefontsize (1+1) EA};
\addplot plot [error bars/.cd, y dir=both, y explicit] coordinates {(32,5.617)+-(0,3.585)(64,7.738)+-(0,5.151)(128,11.6)+-(0,6.623)(256,13.83)+-(0,8.088)(512,22.62)+-(0,10.7)(1024,27.94)+-(0,13.45)(2048,42.28)+-(0,20.16)(4096,57.55)+-(0,24.05)(8192,84.66)+-(0,38.99)(16384,112.5)+-(0,60.03)(32768,171.7)+-(0,78.58)(65536,218.3)+-(0,80.15)(131072,307.2)+-(0,136.5)(262144,465.8)+-(0,179.4)(524288,595.9)+-(0,210.5)(1048576,913.3)+-(0,378.1)(2097152,1239)+-(0,485.9)(4194304,1670)+-(0,647.3)};
\addlegendentry{\thefontsize RLS};
\end{axis}
\end{tikzpicture}
\caption{Mean runtimes and their standard deviation of different algorithms on \onemax with initial Hamming distance $D$ from the optimum equal to $\log(n+1)$ in expectation.
}
\label{exp:log}
\end{figure}

We can also see that the logarithmically capped self-adjusting version, although initially looking well, starts to lose ground when the problem size grows. For $n=2^{22}$ it has roughly the same running time as the \ollga with $\beta \le 2.3$. To see whether this effect is stronger when the algorithm
starts closer to the optimum, we also conducted the series of experiments when the initial distance to the optimum being only logarithmic. The results
are presented in Fig.~\ref{exp:log}. The logarithmically capped version loses already
to $\beta=2.5$ this time, indicating that the fast \ollga is faster close to the optimum than that.

In order to understand better how different choices for $\beta$ behave in practice when the starting point also varies,
we conducted additional experiments with problem size $n=2^{22}$, but with expected initial distances $D$ equal to $2^i$ for $i \in [0..21]$.
We also normalize all the expected running times by $\sqrt{nD}$, but this time we vary $D$.
The results are presented in Fig.~\ref{exp:comp}, where the results are averaged over 10 runs for distances between $2^9$ and $2^{20}$ due to the lack of computational budget. At distances smaller than $2^{12}$ the smaller $\beta > 2$ perform noticeably better, as specified in Table~\ref{tbl:runtime},
however for larger distances the constant factors start to influence the picture: for instance, $\beta = 2.1$ is outperformed by $\beta = 2.3$ at distances greater than $2^{13}$.


We also included in this figure a few algorithms with $\beta < 2$, namely $\beta\in\{1.5,1.7,1.9\}$, which have a distribution upper bound of $\sqrt{n}$, for which running times are averaged over 100 runs.
From Fig.~\ref{exp:comp} we can see that the running time of these algorithms increases with decreasing $\beta$ just as in Table~\ref{tbl:runtime} for comparatively large distances ($2^{12}$ and up),
however for smaller distances their order is reversed, which shows that constant factors still play a significant role.


\begin{figure}[!t]
\begin{tikzpicture}
\begin{axis}[width=0.7\linewidth, height=0.37\textheight, xmode=log, log base x=2, ymode=log, grid=major, ymin=1.2,
             xlabel={Distance to optimum $D$}, ylabel={Evaluations / $\sqrt{nD}$},
             legend pos=outer north east, cycle list name=myplotcycle,
             every axis plot/.append style={very thick}]
             \addplot plot [error bars/.cd, y dir=both, y explicit] coordinates {(1,264.9)+-(0,332)(2,322)+-(0,227.1)(4,330.9)+-(0,226.1)(8,319.1)+-(0,148.5)(16,269.9)+-(0,102.5)(32,213.6)+-(0,63.45)(64,189.8)+-(0,57.42)(128,154.8)+-(0,45.47)(256,117.4)+-(0,24.22)(512,100.5)+-(0,18.98)(1024,79.52)+-(0,21.8)(2048,60.15)+-(0,8.884)(4096,47.18)+-(0,7.378)(8192,35.96)+-(0,1.791)(16384,30.31)+-(0,2.745)(32768,26.16)+-(0,2.308)(65536,23.14)+-(0,0.9264)(131072,20.46)+-(0,0.7859)(262144,18.25)+-(0,0.9256)(524288,16.01)+-(0,0.4698)(1048576,13.65)+-(0,0.4906)(2097152,10.88)+-(0,0.3404)};
             \addlegendentry{\thefontsize $\lambda\in [1..2\ln (n + 1)]$};
             \addplot plot [error bars/.cd, y dir=both, y explicit] coordinates {(1,12.62)+-(0,9.758)(2,13.77)+-(0,6.508)(4,13.09)+-(0,4.128)(8,14.16)+-(0,3.713)(16,14.15)+-(0,2.753)(32,15.1)+-(0,2.202)(64,15.05)+-(0,1.622)(128,15.45)+-(0,1.416)(256,15.58)+-(0,0.8443)(512,15.54)+-(0,0.6976)(1024,16.17)+-(0,0.6433)(2048,15.85)+-(0,0.2113)(4096,15.98)+-(0,0.2897)(8192,15.74)+-(0,0.2668)(16384,15.79)+-(0,0.1366)(32768,15.63)+-(0,0.08493)(65536,15.53)+-(0,0.06992)(131072,15.15)+-(0,0.04679)(262144,14.55)+-(0,0.04327)(524288,13.44)+-(0,0.03713)(1048576,11.69)+-(0,0.01774)(2097152,9.47)+-(0,0.01571)};
             \addlegendentry{\thefontsize $\lambda\in [1..n]$};
             \addplot plot [error bars/.cd, y dir=both, y explicit] coordinates {(1,39.57)+-(0,37.02)(2,42.9)+-(0,32.11)(4,58.14)+-(0,42.17)(8,56.6)+-(0,32.39)(16,65.96)+-(0,33.41)(32,60.9)+-(0,21.94)(64,68.93)+-(0,47.78)(128,66.4)+-(0,30.81)(256,63.62)+-(0,24.77)(512,70.48)+-(0,29.94)(1024,56.14)+-(0,9.314)(2048,63.99)+-(0,19.13)(4096,54.43)+-(0,4.721)(8192,51.59)+-(0,3.185)(16384,55.81)+-(0,17.02)(32768,51.57)+-(0,7.501)(65536,47.25)+-(0,3.055)(131072,42.9)+-(0,2.971)(262144,40.92)+-(0,3.04)(524288,38.36)+-(0,2.001)(1048576,33.5)+-(0,1.294)};
             \addlegendentry{\thefontsize $\lambda\sim\text{pow}(2.1)$};
             \addplot plot [error bars/.cd, y dir=both, y explicit] coordinates {(1,61.71)+-(0,56.3)(2,85.62)+-(0,52.87)(4,92.74)+-(0,65.45)(8,99.77)+-(0,41.04)(16,99.07)+-(0,35.14)(32,96.28)+-(0,23.09)(64,95.06)+-(0,25.12)(128,86.99)+-(0,18.36)(256,80.41)+-(0,13.26)(512,73.91)+-(0,6.717)(1024,67.29)+-(0,4.111)(2048,65.25)+-(0,8.738)(4096,56.84)+-(0,4.415)(8192,50.35)+-(0,1.93)(16384,45.94)+-(0,1.039)(32768,43.72)+-(0,3.304)(65536,37.39)+-(0,0.8303)(131072,33.79)+-(0,1.423)(262144,30.26)+-(0,1.39)(524288,26.01)+-(0,1.37)(1048576,22.31)+-(0,0.6869)};
             \addlegendentry{\thefontsize $\lambda\sim\text{pow}(2.3)$};
             \addplot plot [error bars/.cd, y dir=both, y explicit] coordinates {(1,178)+-(0,171.1)(2,195)+-(0,134.5)(4,211)+-(0,95.85)(8,203.8)+-(0,77.96)(16,199.3)+-(0,68.14)(32,182.1)+-(0,40.47)(64,173.9)+-(0,39.8)(128,149.9)+-(0,23.29)(256,136.1)+-(0,15.88)(512,126.5)+-(0,18.64)(1024,107)+-(0,13.44)(2048,88.9)+-(0,6.161)(4096,77.96)+-(0,3.7)(8192,69.1)+-(0,2.38)(16384,57.96)+-(0,1.518)(32768,49.49)+-(0,1.149)(65536,41.61)+-(0,1.013)(131072,35.6)+-(0,0.8576)(262144,30.23)+-(0,0.4872)(524288,25.17)+-(0,0.3194)(1048576,20.61)+-(0,0.2659)(2097152,16.47)+-(0,0.2419)};
             \addlegendentry{\thefontsize $\lambda\sim\text{pow}(2.5)$};
             \addplot plot [error bars/.cd, y dir=both, y explicit] coordinates {(1,410)+-(0,363.6)(2,469.2)+-(0,411.4)(4,486.7)+-(0,286)(8,466.6)+-(0,229.9)(16,422.7)+-(0,128.1)(32,376.2)+-(0,113.7)(64,328.6)+-(0,73.42)(128,285.2)+-(0,44.4)(256,233.6)+-(0,32.32)(512,202.3)+-(0,16.15)(1024,181.8)+-(0,21.49)(2048,145.1)+-(0,9.932)(4096,121.5)+-(0,9.805)(8192,101.3)+-(0,8.064)(16384,80.32)+-(0,3.482)(32768,64.81)+-(0,2.898)(65536,52.66)+-(0,1.661)(131072,42.47)+-(0,1.155)(262144,35.84)+-(0,1.463)(524288,28.65)+-(0,0.7598)(1048576,22.43)+-(0,0.3778)(2097152,17.54)+-(0,0.4512)};
             \addlegendentry{\thefontsize $\lambda\sim\text{pow}(2.7)$};
             \addplot plot [error bars/.cd, y dir=both, y explicit] coordinates {(1,800.3)+-(0,731.8)(2,919.1)+-(0,612.5)(4,929.5)+-(0,562.7)(8,806.2)+-(0,349.4)(16,800.7)+-(0,307.4)(32,686.6)+-(0,190.4)(64,582.2)+-(0,121.1)(128,501.9)+-(0,97.34)(256,417.8)+-(0,78.1)(512,348)+-(0,94.45)(1024,263.9)+-(0,30.92)(2048,210.3)+-(0,16.2)(4096,174.9)+-(0,16.07)(8192,143.9)+-(0,11.88)(16384,110.7)+-(0,8.234)(32768,87.55)+-(0,6.231)(65536,67.76)+-(0,3.129)(131072,55.45)+-(0,2.277)(262144,43.26)+-(0,1.602)(524288,34.35)+-(0,1.54)(1048576,27)+-(0,1.511)(2097152,20.05)+-(0,0.9004)};
             \addlegendentry{\thefontsize $\lambda\sim\text{pow}(2.9)$};
             \addplot plot [error bars/.cd, y dir=both, y explicit] coordinates {(1,3205)+-(0,3224)(2,3541)+-(0,2506)(4,3764)+-(0,2187)(8,3354)+-(0,1420)(16,3115)+-(0,1134)(32,2458)+-(0,863.4)(64,2172)+-(0,559)(128,1608)+-(0,387.9)(256,1362)+-(0,277)(512,1092)+-(0,209.3)(1024,804.4)+-(0,123.7)(2048,654.4)+-(0,166.1)(4096,521.4)+-(0,111.5)(8192,366.4)+-(0,38.7)(16384,280.9)+-(0,45.71)(32768,208.9)+-(0,17.36)(65536,157.1)+-(0,18.68)(131072,118.4)+-(0,9.369)(262144,88.11)+-(0,9.601)(524288,67.61)+-(0,5.27)(1048576,48.37)+-(0,4.963)(2097152,36.01)+-(0,3.485)};
             \addlegendentry{\thefontsize (1+1) EA};
             \addplot plot [error bars/.cd, y dir=both, y explicit] coordinates {(1,1967)+-(0,1849)(2,2371)+-(0,1698)(4,1957)+-(0,1059)(8,1946)+-(0,863.5)(16,1793)+-(0,686.2)(32,1578)+-(0,492.5)(64,1189)+-(0,319.6)(128,959.7)+-(0,237.4)(256,768.2)+-(0,140.7)(512,585.3)+-(0,71.89)(1024,478.1)+-(0,91.47)(2048,353)+-(0,46.02)(4096,306.9)+-(0,61.45)(8192,221.2)+-(0,22.7)(16384,169.7)+-(0,17.45)(32768,124.4)+-(0,12.99)(65536,92.32)+-(0,7.695)(131072,67.96)+-(0,7.38)(262144,54.29)+-(0,4.651)(524288,38.6)+-(0,2.57)(1048576,28.79)+-(0,2.669)(2097152,21.43)+-(0,2.033)};
             \addlegendentry{\thefontsize RLS};
             \addplot plot [error bars/.cd, y dir=both, y explicit] coordinates {(1,18.49)+-(0,16.74)(2,16.18)+-(0,11.74)(4,21.91)+-(0,10.91)(8,20.47)+-(0,7.326)(16,22.15)+-(0,5.777)(32,22.83)+-(0,4.455)(64,22.9)+-(0,3.124)(128,25.07)+-(0,2.55)(256,26.93)+-(0,2.031)(512,28.79)+-(0,1.515)(1024,31.62)+-(0,1.045)(2048,34.28)+-(0,0.8975)(4096,37.75)+-(0,0.644)(8192,41.71)+-(0,0.5644)(16384,46.32)+-(0,0.5121)};
             \addlegendentry{\thefontsize $\lambda\sim\text{pow}(1.5)*$};
             \addplot plot [error bars/.cd, y dir=both, y explicit] coordinates {(1,21.45)+-(0,19.56)(2,24.08)+-(0,16.84)(4,25.77)+-(0,12.91)(8,26.34)+-(0,10.01)(16,25.43)+-(0,7.23)(32,26.46)+-(0,5.455)(64,27.77)+-(0,3.992)(128,28.27)+-(0,3.316)(256,29.4)+-(0,2.23)(512,30.37)+-(0,1.764)(1024,31.48)+-(0,1.249)(2048,33.04)+-(0,1.022)(4096,34.76)+-(0,0.7229)(8192,36.77)+-(0,0.6249)(16384,38.82)+-(0,0.4412)};
             \addlegendentry{\thefontsize $\lambda\sim\text{pow}(1.7)*$};
             \addplot plot [error bars/.cd, y dir=both, y explicit] coordinates {(1,30.71)+-(0,27.96)(2,32.74)+-(0,21.36)(4,34.59)+-(0,16.67)(8,36.1)+-(0,14.02)(16,37.89)+-(0,11.29)(32,35.99)+-(0,7.03)(64,36.42)+-(0,6.298)(128,35.41)+-(0,4.655)(256,35.54)+-(0,2.992)(512,35.57)+-(0,2.576)(1024,34.81)+-(0,1.684)(2048,35.27)+-(0,1.163)(4096,35.11)+-(0,0.823)(8192,35.28)+-(0,0.524)(16384,35.29)+-(0,0.468)};
             \addlegendentry{\thefontsize $\lambda\sim\text{pow}(1.9)*$};
\end{axis}
\end{tikzpicture}
\caption{Mean runtimes and their standard deviation of different algorithms on \onemax with problem size $n=2^{22}$ and with initial Hamming distances
of the form $D = 2^i$ for $0 \le i \le 21$.
The starred versions of the fast \ollga have a distribution upper bound of $\sqrt{n}$.
}
\label{exp:comp}
\end{figure}

We also considered the variation of the \ollga with population size, mutation strength and crossover bias chosen independently from heavy-tailed distributions, proposed in~\cite{AntipovBD23}.
This algorithm is able to mimic local searches, but in the same time it is capable of solving hard multimodal problems within only a polynomial factor of the optimal performance.
To determine whether this algorithm can benefit from having a good solution on start, we performed additional experiments. While our problem is still \onemax, we consider starting at a distance of $\sqrt{n}$
from the optimum, similarly to the above.

The distribution parameters are as follows.
The population size is controlled by $\beta_{\lambda}$, which should be at least 2 for good performance according to~\cite{AntipovBD23}, so we choose it from $\beta_{\lambda} = \{2.0, 2.2, 2.4, 2.6, 2.8, 3.0, 3.2\}$.
The mutation strength is controlled by $\beta_p$, and the crossover bias is controlled by $\beta_c$, where both are to be selected from $[1;3]$.
Similarly to~\cite{AntipovBD23}, we choose them to be equal and denote $\beta_{pc} = \beta_p = \beta_c$.
In our experiments, $\beta_{pc}$ takes values from $\{1.0, 1.2, 1.4, 1.6, 1.8, 2.0, 2.2\}$ to cover the interesting part of the interval.

\begin{figure}[!t]
\begin{tikzpicture}
    \begin{loglogaxis}[width=\linewidth, height=0.35\textheight,
                       legend pos=north west, legend columns=2, ymin=0.7,
                       cycle list name=threeheavycycle, every axis plot/.append style={very thick},
                       grid=major, log base x=2, enlarge x limits=0.05,
                       enlarge y limits=false, ymin=1, ymax=30,
                       xlabel={Problem size $n$}, ylabel={Runtime / $n \ln n$}]
        \addplot plot [error bars/.cd, y dir=both, y explicit] coordinates {(8,3.526)+-(0,3.172)(16,4.672)+-(0,3.112)(32,5.272)+-(0,3.277)(64,6.223)+-(0,2.651)(128,6.918)+-(0,2.454)(256,8.258)+-(0,2.639)(512,9.677)+-(0,2.676)(1024,11.03)+-(0,3.016)(2048,12.28)+-(0,2.382)(4096,13.15)+-(0,2.73)(8192,14.52)+-(0,2.335)(16384,15.31)+-(0,2.522)};
        \addlegendentry{$\beta_{pc}=1.0$};
        \addplot plot [error bars/.cd, y dir=both, y explicit] coordinates {(8,3.637)+-(0,2.738)(16,3.533)+-(0,1.738)(32,3.221)+-(0,1.302)(64,3.44)+-(0,1.378)(128,3.488)+-(0,1.185)(256,3.126)+-(0,1.006)(512,3.096)+-(0,0.8529)(1024,2.943)+-(0,0.7909)(2048,2.886)+-(0,0.6717)(4096,2.801)+-(0,0.5509)(8192,2.562)+-(0,0.5344)(16384,2.555)+-(0,0.3928)};
        \addlegendentry{$\beta_{pc}=1.8$};
        \addplot plot [error bars/.cd, y dir=both, y explicit] coordinates {(8,3.279)+-(0,2.623)(16,4.547)+-(0,3.057)(32,4.529)+-(0,2.427)(64,5.357)+-(0,2.689)(128,5.86)+-(0,2.192)(256,6.047)+-(0,1.924)(512,6.115)+-(0,1.731)(1024,6.862)+-(0,1.712)(2048,7.198)+-(0,1.572)(4096,7.593)+-(0,1.448)(8192,7.84)+-(0,1.433)(16384,7.834)+-(0,1.204)};
        \addlegendentry{$\beta_{pc}=1.2$};
        \addplot plot [error bars/.cd, y dir=both, y explicit] coordinates {(8,2.908)+-(0,2.158)(16,3.71)+-(0,1.881)(32,3.175)+-(0,1.837)(64,3.016)+-(0,1.266)(128,3.212)+-(0,1.226)(256,2.953)+-(0,1.001)(512,2.708)+-(0,0.7167)(1024,2.498)+-(0,0.5939)(2048,2.4)+-(0,0.5207)(4096,2.283)+-(0,0.4545)(8192,2.128)+-(0,0.4184)(16384,2.058)+-(0,0.3129)};
        \addlegendentry{$\beta_{pc}=2.0$};
        \addplot plot [error bars/.cd, y dir=both, y explicit] coordinates {(8,3.662)+-(0,3.231)(16,4.061)+-(0,2.472)(32,3.781)+-(0,1.838)(64,4.59)+-(0,2.174)(128,4.835)+-(0,1.875)(256,4.813)+-(0,1.52)(512,4.769)+-(0,1.445)(1024,4.858)+-(0,1.128)(2048,4.771)+-(0,1.091)(4096,4.973)+-(0,0.9351)(8192,4.906)+-(0,0.8357)(16384,4.738)+-(0,0.7041)};
        \addlegendentry{$\beta_{pc}=1.4$};
        \addplot plot [error bars/.cd, y dir=both, y explicit] coordinates {(8,3.5574455)+=(0,3.6070653)-=(0,2.7574455)(16,3.424)+-(0,2.118)(32,3.207)+-(0,1.898)(64,2.849)+-(0,1.067)(128,2.565)+-(0,0.882)(256,2.572)+-(0,0.9128)(512,2.38)+-(0,0.7033)(1024,2.266)+-(0,0.6134)(2048,2.146)+-(0,0.5176)(4096,1.975)+-(0,0.3569)(8192,1.865)+-(0,0.3215)(16384,1.816)+-(0,0.2849)};
        \addlegendentry{$\beta_{pc}=2.2$};
        \addplot plot [error bars/.cd, y dir=both, y explicit] coordinates {(8,3.6)+-(0,2.553)(16,3.609)+-(0,2.164)(32,3.298)+-(0,1.899)(64,3.909)+-(0,1.87)(128,3.972)+-(0,1.523)(256,3.899)+-(0,1.268)(512,3.661)+-(0,0.8882)(1024,3.738)+-(0,1.14)(2048,3.793)+-(0,0.8543)(4096,3.458)+-(0,0.719)(8192,3.461)+-(0,0.58)(16384,3.27)+-(0,0.5761)};
        \addlegendentry{$\beta_{pc}=1.6$};
        \addplot plot [error bars/.cd, y dir=both, y explicit] coordinates {(8,1.638)+-(0,1.188)(16,1.779)+-(0,1.063)(32,1.759)+-(0,1.01)(64,1.633)+-(0,0.6288)(128,1.582)+-(0,0.7011)(256,1.684)+-(0,0.6456)(512,1.558)+-(0,0.4811)(1024,1.539)+-(0,0.5533)(2048,1.592)+-(0,0.4723)(4096,1.555)+-(0,0.4819)(8192,1.562)+-(0,0.3582)(16384,1.502)+-(0,0.4143)};
        \addlegendentry{(1+1) EA};
    \end{loglogaxis}
\end{tikzpicture}
\caption{Running times of the three-distribution \ollga on \onemax
         starting from a point at distance $\sqrt{n}$ from the optimum,
         normalized by $n \ln n$,
         for different $\beta_{pc} = \beta_p = \beta_c$ and $\beta_\lambda=2.8$
         depending on the problem size $n$.
         The expected running times of \oea, also starting from the same distance, are given for comparison.}
\label{plot:omsqrt:pc}
\end{figure}
  
\begin{figure}[!t]
\begin{tikzpicture}
    \begin{semilogyaxis}[width=\linewidth, height=0.35\textheight,
                         legend pos=south east, legend columns=2,
                         cycle list name=threeheavycycle, every axis plot/.append style={very thick},
                         grid=major,
                         xlabel={$\beta_\lambda$}, ylabel={Runtime}]
        \addplot plot [error bars/.cd, y dir=both, y explicit] coordinates {(2.0,3730025)+=(0,12404850)-=(0,3727025)(2.2,1.46e+06)+-(0,4.171e+05)(2.4,1.598e+06)+-(0,1.778e+05)(2.6,1.915e+06)+-(0,2.684e+05)(2.8,2.434e+06)+-(0,4.01e+05)(3.0,2.977e+06)+-(0,4.84e+05)(3.2,3.572e+06)+-(0,6.453e+05)};
        \addlegendentry{$\beta_{pc}=1.0$};
        \addplot plot [error bars/.cd, y dir=both, y explicit] coordinates {(2.0,3.194e+05)+-(0,2.022e+05)(2.2,2.334e+05)+-(0,4.078e+04)(2.4,2.616e+05)+-(0,3.892e+04)(2.6,3.226e+05)+-(0,4.596e+04)(2.8,4.062e+05)+-(0,6.245e+04)(3.0,4.852e+05)+-(0,8.635e+04)(3.2,5.746e+05)+-(0,1.062e+05)};
        \addlegendentry{$\beta_{pc}=1.8$};
        \addplot plot [error bars/.cd, y dir=both, y explicit] coordinates {(2.0,1.111e+06)+-(0,1.104e+06)(2.2,7.195e+05)+-(0,1.795e+05)(2.4,8.169e+05)+-(0,1.274e+05)(2.6,1.007e+06)+-(0,1.345e+05)(2.8,1.246e+06)+-(0,1.913e+05)(3.0,1.506e+06)+-(0,2.487e+05)(3.2,1.738e+06)+-(0,3.382e+05)};
        \addlegendentry{$\beta_{pc}=1.2$};
        \addplot plot [error bars/.cd, y dir=both, y explicit] coordinates {(2.0,2.683e+05)+-(0,2.132e+05)(2.2,1.979e+05)+-(0,8.655e+04)(2.4,2.135e+05)+-(0,2.945e+04)(2.6,2.618e+05)+-(0,3.37e+04)(2.8,3.272e+05)+-(0,4.975e+04)(3.0,3.832e+05)+-(0,6.05e+04)(3.2,4.719e+05)+-(0,8.746e+04)};
        \addlegendentry{$\beta_{pc}=2.0$};
        \addplot plot [error bars/.cd, y dir=both, y explicit] coordinates {(2.0,1277287.3)+=(0,3673519.4)-=(0,1274287.3)(2.2,5.035e+05)+-(0,4.99e+05)(2.4,4.867e+05)+-(0,6.651e+04)(2.6,5.931e+05)+-(0,7.884e+04)(2.8,7.533e+05)+-(0,1.119e+05)(3.0,9.085e+05)+-(0,1.525e+05)(3.2,1.097e+06)+-(0,2.139e+05)};
        \addlegendentry{$\beta_{pc}=1.4$};
        \addplot plot [error bars/.cd, y dir=both, y explicit] coordinates {(2.0,275583.58)+=(0,297007.14)-=(0,272583.58)(2.2,1.688e+05)+-(0,4.403e+04)(2.4,1.867e+05)+-(0,3.903e+04)(2.6,2.271e+05)+-(0,3.342e+04)(2.8,2.887e+05)+-(0,4.529e+04)(3.0,3.484e+05)+-(0,6.077e+04)(3.2,4.117e+05)+-(0,8.286e+04)};
        \addlegendentry{$\beta_{pc}=2.2$};                   
        \addplot plot [error bars/.cd, y dir=both, y explicit] coordinates {(2.0,469652.54)+=(0,728872.51)-=(0,466652.54)(2.2,3.3e+05)+-(0,2.133e+05)(2.4,3.488e+05)+-(0,5.998e+04)(2.6,4.165e+05)+-(0,6.775e+04)(2.8,5.2e+05)+-(0,9.159e+04)(3.0,6.488e+05)+-(0,1.104e+05)(3.2,7.529e+05)+-(0,1.403e+05)};
        \addlegendentry{$\beta_{pc}=1.6$};
    \end{semilogyaxis}
\end{tikzpicture}
\caption{Running times of the three-distribution \ollga on \onemax
         starting from a point at distance $\sqrt{n}$ from the optimum,
         for $n=2^{14}$ and different $\beta_{pc} = \beta_p = \beta_c$
         depending on $\beta_\lambda$.}
\label{plot:omsqrt:lambda}
\end{figure}

Figure~\ref{plot:omsqrt:pc} presents the running times for this setup, normalized by $n \log n$, with varying $\beta_{pc}$ and $\beta_{\lambda} = 2.8$,
for initialization at the Hamming distance of $\sqrt{n}$ from the optimum. We also plot the results for the \oea.
As the latter is known to traverse this distance in time $\Theta(n \log n)$, its plot is horizontal in Figure~\ref{plot:omsqrt:pc}.
The \ollga, on the other hand, shows a different behavior, as the plots for $\beta_{pc} \ge 1.4$ start descending at some $n$ that grows as $\beta_{pc}$ decreases.
We speculate that this will also happen for other values of $\beta_{pc}$ when $n$ is large enough.
This means that at small distances to the optimum, where the single-bit flipping mode employed by the \oea is not very efficient,
even the generalized version of the \ollga with parameters chosen ``lazily'' from heavy-tailed distributions benefits from starting close to the optimum.

Figure~\ref{plot:omsqrt:lambda} investigates the interplay of the algorithm's parameters when $\beta_{\lambda}$ also varies. One can see that larger $\beta_{pc}$ are still beneficial for \onemax
even when the algorithm starts from a good solution, because the probability of changing only one bit in the crossover phase is bigger. However, larger population sizes are better,
which a rather small best value $\beta_{\lambda} = 2.2$ suggests. Note that when the algorithm starts from a random solution, values closer to $2.5$ are better. As the standard \ollga progresses faster
with larger population sizes close to the optimum, we may conjecture that it is true for the three-distribution version as well.

\section{Conclusion}
In this paper we proposed a new notion of the fixed-start runtime analysis, which in some sense complements the fixed-target notion. Among the first results in this direction we observed that different algorithms profit differently from having an access to a solution close to the optimum.

The performance of all observed algorithms, however, is far from the theoretical lower bound. Hence, we are still either to find the EAs which can benefit from good initial solutions or to prove a stronger lower bounds for unary and binary algorithms.

\section*{Acknowledgements}

This work was supported by a public grant as part of the Investissement d'avenir project, reference ANR-11-LABX-0056-LMH, LabEx LMH, in a joint call with Gaspard Monge Program for optimization, operations research and their interactions with data sciences.

\bibliographystyle{alpha}
\bibliography{ich_master,alles_ea_master,extra} 

}

\end{document}